\documentclass{article}
\usepackage{graphicx} 
\usepackage{subcaption}

\usepackage{amsfonts,amsmath}
\usepackage{fullpage}
\usepackage{hyperref}
\usepackage{cleveref}
\usepackage[dvipsnames]{xcolor}
\RequirePackage[style      = numeric-comp,
                giveninits=true,
                uniquename=false,
                uniquelist=false,
                sorting    = nyt,
                sortcites  = true,
                autopunct  = true,
                hyperref   = true,
                abbreviate = false,
                backref    = true,
                backend    = biber,
                maxbibnames = 99,
                isbn=false,url=false,eprint=false,
                defernumbers = true,
                citereset  = section]{biblatex}
\usepackage{subcaption}
\usepackage{multirow}
\usepackage{algorithm}
\usepackage{algpseudocode}
\usepackage{tikz}
\usepackage{amsthm}

\newtheorem{fact}{Statement}

\newtheorem{prop}{Proposition}

\Crefname{fact}{Statement}{Statements}
\usepackage{siunitx}
\DeclareMathOperator*{\argmin}{arg\,min}

\usepackage{amsmath,amsfonts,bm}

\usepackage{array}

\def\eqref#1{equation~\ref{#1}}
\def\1{\bm{1}}

\DeclareMathAlphabet{\mathsfit}{\encodingdefault}{\sfdefault}{m}{sl}
\SetMathAlphabet{\mathsfit}{bold}{\encodingdefault}{\sfdefault}{bx}{n}

\newcommand{\bfA}{{\bf A}}

\newcommand{\bfC}{{\bf C}}
\newcommand{\bfD}{{\bf D}}
\newcommand{\bfE}{{\bf E}}

\newcommand{\bfI}{{\bf I}}

\newcommand{\bfL}{{\bf L}}
\newcommand{\bfM}{{\bf M}}

\newcommand{\bfP}{{\bf P}}

\newcommand{\bfS}{{\bf S}}

\newcommand{\bfU}{{\bf U}}
\newcommand{\bfV}{{\bf V}}
\newcommand{\bfW}{{\bf W}}
\newcommand{\bfX}{{\bf X}}
\newcommand{\bfY}{{\bf Y}}

\newcommand{\bfx}{{\bf x}}
\newcommand{\bfy}{{\bf y}}
\newcommand{\bfu}{{\bf u}}

\newcommand{\bfv}{{\bf v}}
\newcommand{\bfw}{{\bf w}}
\newcommand{\bfz}{{\bf z}}
\newcommand{\bft}{{\bf t}}

\newcommand{\bfepsilon}{{\boldsymbol \epsilon}}
\newcommand{\bftheta}{{\boldsymbol \theta}}

\def\mydefb#1{\expandafter\def\csname bf#1\endcsname{\mathbf{#1}}}
\def\mydefallb#1{\ifx#1\mydefallb\else\mydefb#1\expandafter\mydefallb\fi}
\mydefallb aAbBcCdDeEfFgGhHiIjJkKlLmMnNoOpPqQrRsStTuUvVwWxXyYzZ0123456789\mydefallb

\def\mydefgreek#1{\expandafter\def\csname bf#1\endcsname{\text{\boldmath$\mathbf{\csname #1\endcsname}$}}}
\def\mydefallgreek#1{\ifx\mydefallgreek#1\else\mydefgreek{#1}%
   \lowercase{\mydefgreek{#1}}\expandafter\mydefallgreek\fi}
\mydefallgreek {alpha}{Alpha}{beta}{Beta}{gamma}{Gamma}{delta}{Delta}{epsilon}{Epsilon}{varepsilon}{Varepsilon}{zeta}{Zeta}{eta}{Eta}{theta}{Theta}{iota}{Iota}{kappa}{Kappa}{lambda}{Lambda}{mu}{Mu}{nu}{Nu}{omicron}{Omicron}{pi}{Pi}{rho}{Rho}{sigma}{Sigma}{tau}{Tau}{upsilon}{Upsilon}{varphi}{phi}{Phi}{xi}{Xi}{chi}{Chi}{psi}{Psi}{omega}{Omega}\mydefallgreek

\def\mydefb#1{\expandafter\def\csname bb#1\endcsname{\mathbb{#1}}}
\def\mydefallb#1{\ifx#1\mydefallb\else\mydefb#1\expandafter\mydefallb\fi}
\mydefallb aAbBcCdDeEfFgGhHiIjJkKlLmMnNoOpPqQrRsStTuUvVwWxXyYzZ\mydefallb

\def\mydefb#1{\expandafter\def\csname cal#1\endcsname{\mathcal{#1}}}
\def\mydefallb#1{\ifx#1\mydefallb\else\mydefb#1\expandafter\mydefallb\fi}
\mydefallb aAbBcCdDeEfFgGhHiIjJkKlLmMnNoOpPqQrRsStTuUvVwWxXyYzZ\mydefallb

  \RequirePackage{listings}
  \RequirePackage{tikz}
    \usetikzlibrary{arrows}
    \usetikzlibrary{arrows.meta}
    \usetikzlibrary{backgrounds}
    \usetikzlibrary{calc}
    \usetikzlibrary{decorations.pathreplacing}
    \usetikzlibrary{tikzmark}
    \usetikzlibrary{shapes.geometric}
    \usetikzlibrary{positioning}
    \usetikzlibrary{decorations}
    \usetikzlibrary{fit}
    \usetikzlibrary{shapes}
  \RequirePackage{tikz-3dplot}
  \RequirePackage{pgfplots}
    \pgfplotsset{compat=1.16}
    \pgfdeclarelayer{background}
    \pgfdeclarelayer{foreground}
    \pgfsetlayers{background,main,foreground}
  \RequirePackage{tkz-euclide}

\definecolor{matlab1}{RGB}{0,    114,  189} % Matlabs default plot colors
\definecolor{matlab2}{RGB}{217,   83,   25}
\definecolor{matlab3}{RGB}{237,  177,   32}
\definecolor{matlab4}{RGB}{126,   47,  142}
\definecolor{matlab5}{RGB}{119,  172,   48}
\definecolor{matlab6}{RGB}{77,   190,  238}
\definecolor{matlab7}{RGB}{162,   20,   47}
\renewcommand{\t} {^{\top}}                                % transpose, e.g. $A\t$
\newcommand{\diag} [1]  {{\rm diag\!}\left( #1 \right)}    % diagonal of a matrix
\renewcommand{\ker}[1]  {{\rm ker\!}\left( #1 \right)}     % kernel of a function
\newcommand{\vertbar}{\rule[-1ex]{0.5pt}{2.5ex}}           % vertical bar
\newcommand{\horzbar}{\rule[.5ex]{2.5ex}{0.5pt}}           % horizontal bar

\usepackage[pagewise]{lineno}

\title{Latent Space Inference via Paired Autoencoders}
\author{Emma Hart, Bas Peters, Julianne Chung, and Matthias Chung}
\date{\today}

\begin{document}

\maketitle

\begin{abstract}
This work describes a novel data-driven latent space inference framework built on paired autoencoders to handle observational inconsistencies when solving inverse problems.  Our approach uses two autoencoders, one for the parameter space and one for the observation space, connected by learned mappings between the autoencoders' latent spaces.  These mappings enable a surrogate for regularized inversion and optimization in low-dimensional, informative latent spaces. Our flexible framework can work with partial, noisy, or out-of-distribution data, all while maintaining consistency with the underlying physical models. The paired autoencoders enable reconstruction of corrupted data, and then use the reconstructed data for parameter estimation, which produces more accurate reconstructions compared to paired autoencoders alone and end-to-end encoder-decoders of the same architecture, especially in scenarios with data inconsistencies.  We demonstrate our approaches on two imaging examples in medical tomography and geophysical seismic-waveform inversion, but the described approaches are broadly applicable to a variety of inverse problems in scientific and engineering applications.
\end{abstract}

\section{Introduction and Background}

Many scientific and engineering endeavors involve understanding the properties of a system through indirect measurements. This process often leads to what are known as \emph{inverse problems}. Unlike direct or forward problems, where causes are known and effects are predicted, inverse problems aim to infer causes from observed effects \cite{ kaipio2006statistical,engl1996regularization}.  Inference depends, next to others, on the experimental setup of how data are obtained. \bigskip

\noindent\emph{Inverse Problems.} In mathematical terms, let us consider a system described by a mathematical operator $A\colon\mathcal{X} \to \mathcal{Y}$ that maps model parameters (or controls) $\bfx \in \mathcal{X}$ to observations $\bfy \in \mathcal{Y}$ potentially contaminated with some additive unbiased noise $\bfepsilon \in \mathcal{Y}$. The forward operator $A = Q\circ F$ typically consists of solutions  $F\colon \mathcal{X} \to \mathcal{Q}$ to some governing equations  described by a partial or ordinary differential equation (PDE/ODE), which is assumed to be fully determined by its model parameters $\bfx$, and a mapping $Q\colon \mathcal{Q} \to \mathcal{Y}$ onto the observation space $\calY$, i.e.,
\begin{equation}\label{eq:ip}
    A (\bfx) + \bfepsilon = \bfy.
\end{equation}
Given the forward operator $A$, the observations $\bfy$, and some assumptions about the noise, the inverse problem aims to reconstruct the unknown model parameters $\bfx$, often termed as \emph{quantities of interest (QoI)} \cite{hansen2010discrete}.

Inverse problems are numerous and arise in a variety of applications. For instance, in medical imaging, they are used to reconstruct images of internal organs from measurements obtained through techniques like X-ray computed tomography (CT) or magnetic resonance imaging (MRI) \cite{bertero2021introduction,epstein2007introduction}. In geophysics applications, inverse methods are employed to determine the Earth's subsurface structure by analyzing seismic data \cite{zhdanov2002geophysical,haber2014computational}. In non-destructive testing, these techniques are crucial for identifying defects within materials \cite{tanaka1998inverse}. Remote sensing relies heavily on inverse problem-solving to retrieve atmospheric and surface properties from satellite observations \cite{doicu2010numerical,sgattoni2024physics}. More broadly in image processing, tasks like inpainting -- the reconstruction of missing or damaged image regions -- require inferring missing pixel values based on the surrounding image content \cite{guillemot2013image}. \bigskip

\noindent\emph{Challenges.}
A key challenge in inverse problems is their ill-posedness, meaning no model may perfectly explain the data, multiple models may fit the data equally well, or small data perturbations may cause large changes in the estimated model \cite{hadamard1902problemes}. 
Consequently, specialized techniques are essential. These include regularization, which introduces constraints or prior information to stabilize the solution; optimization, which frames the problem as finding the best-fitting model; and statistical methods, which incorporate statistical models of data and unknowns to quantify uncertainty \cite{calvetti2018inverse,hofmann2013regularization}. The study of inverse problems is vibrant and interdisciplinary, drawing upon mathematics, physics, statistics, and computer science. It plays a crucial role in advancing our understanding of complex systems and enabling new technologies across various domains.\bigskip

\noindent\emph{Data-Driven Approaches.} The emergence of data-driven approaches, particularly with the rise of deep neural networks, has significantly impacted the field of inverse problems \cite{arridge2019solving,afkham2021learning}. Traditionally, solving inverse problems relied heavily on analytical methods, numerical simulations, and optimization algorithms. These approaches often require significant computational resources and may struggle with complex or high-dimensional problems. Deep learning offers a powerful alternative by learning the mapping between measurements $\bfy$ and model parameters $\bfx$ directly from data. By training a neural network on a large dataset of paired input-output examples (e.g., measurements and corresponding ground truth models), it can learn to approximate the inverse mapping, i.e., a $\bftheta$-parameterized neural network $\Phi_\theta: \mathcal{Y} \to \mathcal{X}$, effectively bypassing the need for explicit inversion algorithms in some cases.  Methods that learn a map this way are also referred to as \emph{end-to-end inference} approaches. These approaches have shown promising results in various applications, including image reconstruction, medical imaging, and seismic inversion, offering faster and potentially more accurate solutions compared to traditional methods \cite{kim2018geophysical, ongie2020deep,pilozzi2018machine,genedy2023physics,chung2024paired,afkham2024uncertainty}. However, challenges remain, such as the need for large training datasets, ensuring generalization to unseen, partial, or corrupted data, and interpreting the learned representations within neural networks.

\bigskip

\noindent\emph{Partial Observations of Data.} 
A key drawback of using neural networks to learn end-to-end models is their problem-specific nature, necessitating re-training for any modification to the problem setup. 
Standard machine learning methods often falter in reconstructing $\bfx$ from $\bfy_{\text{sub}}=P(\bfy)$, where $P: \calY\to\calY_{\rm sub}$ is an operation that introduces errors in $\bfy_{\text{sub}}$ ($\bfy_{\text{sub}}$ is out-of-distribution of the training data) or data gaps in $\bfy_{\text{sub}} \notin \mathcal{Y}$.
Depending on network architecture and training, previously trained models on full data $\bfy$ become intractable or ineffective, due to the model's performance being highly sensitive to the distribution and characteristics of the training data. When presented with discrepant or incomplete inputs $\bfy_{\text{sub}}$, the model's assumptions are violated, rendering its predictions unreliable. Consequently, one is forced to resort to post hoc methods, such as data imputation or error correction, which can be computationally expensive and may introduce further inaccuracies.  Consider computed tomography: if a data-driven model is trained on sinograms acquired at specific angles, it is likely to perform poorly if applied to data from a patient measured at slightly different angles or with a different number of projections (resulting in a change in the dimensions of the measurement vector $\bfy_{\text{sub}}$) \cite{riis2021computed,afkham2024uncertainty}. In such instances, the direct utilization of pre-trained models proves to be unreliable or even intractable. Analogous difficulties are encountered in the geophysical domain.\bigskip

\noindent\emph{Purpose and Contributions.} 
In this paper, we aim to advance methods using autoencoders to solve inverse problems, extending the scope of their use to include a broader range of problems that arise in application.  We propose a novel latent space inference approach where missing data or other observational inconsistencies can be addressed via optimization in the latent space defined by paired autoencoders. This framework can be viewed as a latent space embedding combined with a data inference technique. Our approach offers an alternative to post hoc interpolation methods for generalization of results.  We illustrate various challenges stemming from inconsistent discretizations, incompleteness, or inaccuracies in input data. While more generally applicable, we focus our study on cases where experimental setups are inconsistent, such as varying or missing angles in computer tomography, or missing observations due to defective sensors in seismic tomography.  Examples are included to demonstrate the flexibility of our approach compared to previous autoencoder methods.

\bigskip

\noindent\emph{Structure.} This work is organized as follows. In \Cref{sec:background}, we provide an overview of variational inverse problems and end-to-end inference. We review encoder-decoder and autoencoder networks and use them to describe paired autoencoder frameworks. In \Cref{sec:lsi}, we present our novel approach, which incorporates paired autoencoders and inference in the latent space, and we include theoretical results. \Cref{sec:numerics} presents numerical experiments that demonstrate the advantages of our approach through comparisons with other methods highlighted for limited-angle or missing-angle tomography and sparse data acquisition for seismic imaging. Finally, \Cref{sec:conclusion} offers concluding remarks and outlines potential avenues for future research.

\section{Background}\label{sec:background}
\noindent\emph{Notation.}  In general, bolded capital letters are used to denote matrices (e.g., $\mathbf{D}$, $\mathbf{E}$) and bold lowercase letters (e.g., $\mathbf{x}$, $\mathbf{y}$) are used for vectors, while un-bolded letters are used for functions, operators, and scalars.  Superscripts denote the kind of network or sample (e.g., ``e" for encoder, ``d" for decoder, ``train" for training data) while subscripts denote what kind of data the object corresponds to (i.e., $\bfz_x$ is a vector in the latent space of an autoencoder trained on $\bfx$ samples).  For readability, all subscripts are kept unbolded.  In general, we use $\bfx_{\rm pred}$ to denote the parameters (QoI) predicted by any method, and $\widehat \bfx$ specifically to denote the parameters predicted by our proposed method, \textit{latent space inference}.

\bigskip

\noindent\emph{Variational Inverse Problems.} For an inverse problem \Cref{eq:ip}, we assume there exists a mapping $\Psi\colon \calY \to \calX$ such that $\Psi(\bfy) = \bfx$. This is the inverse operator, which is often approximated by some classical inversion techniques, denoted $A^{\dagger}(\bfy) \approx \bfx$. An approximate inverse operator$A^{\dagger}$ is typically difficult and computationally challenging to approximate. Inversion depends on the discretization and the challenging task of properly including prior knowledge of $\bfx$. Hence, variational inverse problems are formulated as optimization problems of the form
\begin{equation}\label{eq:varip}
    \min_x \ J(A(\bfx),\bfy) + R(\bfx), 
\end{equation}
where $J\colon \mathcal{X} \times \mathcal{Y} \to \mathbb{R}$ is a data discrepancy loss while $R\colon \mathcal{X} \to \mathbb{R}$ is an appropriate regularization functional encoding prior knowledge on $\bfx$. Letting $\bfx_{\rm pred}$ be a minimizer of \Cref{eq:varip}, then this problem formulation aims at approximating $\bfx$, i.e., $\bfx_{\rm pred} \approx \Psi(\bfy)$. \bigskip

\noindent\emph{End-to-End Inference.} As discussed above, in an end-to-end inference approach, we aim to learn a mapping between two spaces by utilizing a parameterized function, typically a neural network $\Phi_\theta\colon \calY \to \calX$. The network's parameters $\bftheta\in\mathbb{R}^p$ are trained to minimize the discrepancy between the network's output $\bfx_{\rm pred}=\Phi_\theta(\bfy)$ and the true  $\bfx = \Psi(\bfy)$. This approach is theoretically supported by Universal Approximation Theorems \cite{hornik1989multilayer,cybenko1989approximation,leshno1993multilayer}, which suggest that neural networks can approximate any continuous mapping under suitable conditions. The availability of corresponding samples from both the input and target distributions enables the network to learn the mapping. It has been shown that these data-driven methods offer flexibility in capturing complex mappings, benefit from direct data utilization, and allow for end-to-end learning of the entire mapping function. These end-to-end inference approaches encompass a diverse range of architectures, each tailored to specific tasks, including feedforward networks, Convolutional Neural Networks (CNNs), and Residual Networks (ResNets), among many others. \bigskip

\noindent\emph{Encoder-Decoder Networks.} Encoder-decoder networks are a general architectural framework designed to model mappings between input and output spaces through a two-stage process.  The encoder $e_{\theta^e}: \mathcal{Y} \to \mathcal{Z}$ maps high-dimensional input data to a lower-dimensional latent representation, and the decoder network $d_{\theta^d}:\mathcal{Z}\to\mathcal{X}$ generates an output in the target space.  These networks are parameterized by weights $\bftheta^e$ and $\bftheta^d$, which are optimized to minimize a task-specific loss, such as prediction error.  The encoder-decoder structure provides a flexible framework for learning complex, nonlinear relationships between spaces of differing dimensions, and has been widely adopted in areas including image-to-image translation and inverse problems.  Unlike other end-to-end models that learn a direct mapping in a single step, the structured bottleneck of an encoder-decoder can enable more interpretability and regularization. \bigskip

\noindent \emph{Autoencoders.} Autoencoders represent one type of encoder-decoder that is designed to learn efficient representations of data through a self-supervised learning process.  They are made of an \emph{encoder} network $e_{\theta^e}: \mathcal{X} \to \mathcal{Z}$ that maps high-dimensional input data to a lower-dimensional \emph{latent space}, and a \emph{decoder} network $d_{\theta^d}:\mathcal{Z}\to\mathcal{X}$ that reconstructs the original input from the latent representation, with $\mathcal{X} \subset \mathbb{R}^n$, and $\mathcal{Z} \subset \mathbb{R}^\ell$, typically $\ell\ll n$.  The parameters $\bftheta^e$ and $\bftheta^d$ are jointly optimized to minimize a reconstruction loss, typically measuring the discrepancy between input $\bfx$ and reconstruction $d_{\theta^d}(e_{\theta^e}(\bfx))$.  This framework exploits the natural structure of inputs to create a data-specific method of compression that aims to retain the most important input features.  
\bigskip

\noindent \emph{PAIR: Paired Autoencoders for Inference and Regularization.}
PAIR is a data-driven approach for solving inverse problems, where two autoencoders are used to efficiently represent the input and target spaces separately, and optimal mappings are learned between latent spaces. In the PAIR framework, we have autoencoders 
$$
    a_x = d_x \circ e_x \qquad \text{and} \qquad a_y = d_y \circ e_y
$$
where, for better readability, we omit trainable parameters, and subscripts denote what the autoencoder is trained with respect to.  As above, $\mathcal{X}\subset\bbR^n$ is the parameter (or control) space, while $\mathcal{Y}\subset\bbR^q$ is the observation space.  The encoders $e_x:\mathcal{X}\to\mathcal{Z}_\mathcal{X}$ and $e_y:\mathcal{Y}\to\mathcal{Z}_\mathcal{Y}$ compress inputs $\bfx$ to $\mathcal{Z}_\mathcal{X}\subset\bbR^{\ell_x}$ and  $\bfy$ to $\mathcal{Z}_\mathcal{Y}\subset\bbR^{\ell_y}$ respectively, with $\ell_x\ll n$ and $\ell_y \ll q$, while the decoders $d_x:\mathcal{Z}_\mathcal{X}\to\mathcal{X}$ and $d_y:\mathcal{Z}_\mathcal{Y}\to\mathcal{Y}$ reconstruct the original inputs from their latent representations.
By constructing an individual autoencoder for each data instance, we aim to ensure that the key properties of both the data, $\bfy$, and the model, $\bfx$, are encoded, preserving their distinct characteristics.  Additionally, let us assume we have latent space mappings,
\begin{equation}
    m^{\to} \colon \mathcal{Z}_\mathcal{X} \to \mathcal{Z}_\mathcal{Y} \qquad \text{and} \qquad m^{\gets} \colon \mathcal{Z}_\mathcal{Y} \to \mathcal{Z}_\mathcal{X},
\end{equation}
connecting the latent spaces of both autoencoders in each direction.  
We train $a_x$, $a_y$, $m^{\to}$, and $m^{\gets}$ with corresponding $\{\bfy,\bfx\}$ samples to minimize the combined loss \begin{equation}\label{eq:PAIRloss}
    \sum_{\bfy,\bfx} \|a_x(\bfx) - \bfx\|_2^2 + \|a_y(\bfy) - \bfy\|_2^2 + \|(d_x \circ m^{\gets} \circ e_y)(\bfy)-\bfx\|_2^2+ \|(d_y \circ m^{\to} \circ e_x)(\bfx)-\bfy\|_2^2
\end{equation}
We now call $a_x$ and $a_y$ \emph{paired autoencoders}, since the latent mappings enable meaningful inference across the parameter and observation spaces.  For example, a data-driven surrogate mapping of $\Phi$ can be established with 
\begin{equation}\label{eq:PAIRinverse}
d_{x}\circ m^{\gets}\circ e_y \approx \Phi.
\end{equation}
 
These paired autoencoders offer several advantages, including ease of training, high-quality inversion results, computational efficiency, and effective out-of-distribution detection.
The PAIR approach enables forward and inverse surrogate mappings and has been considered for likelihood-free estimation in inverse problems \cite{chung2024paired, chung2025latent, chung2025good}.   
Theoretical connections to Bayes risk and empirical Bayes risk minimization were investigated in \cite{hart2025paired}. To the best of our knowledge, this is the first work to consider the use of PAIR for handling problems with observational inconsistencies such as missing or corrupted data.

\section{Latent Space Inference using PAIR}\label{sec:lsi}
In this section, we describe a new approach for inference that exploits a paired autoencoder framework, coupled with inversion in the latent space, to effectively address inconsistencies in the data. Similar to likelihood-free inference techniques, our approach does not require knowledge of the forward model.  However, the added benefit is that optimization in the latent space enables improved reconstructions and a framework for handling data inconsistencies. Moreover, the general formulation provides a versatile and broadly applicable solution across diverse problem domains.

The crux of the approach lies in the observation that latent spaces of trained autoencoders can offer useful, lower-dimensional spaces for optimization and possible regularization.
If we utilize a paired autoencoder framework, latent space inference (LSI) can be stated as
\begin{equation}\label{eq:latentcomp_zx}
    \hat \bfz_x\in \argmin_{\bfz_x} \left\| (d_y \circ m^{\to})(\bfz_x)- \bfy \right\| \qquad \text{with} \quad \hat \bfx = d_x(\hat \bfz_x).
\end{equation}
Rather than optimize in the parameter latent space, we could also (or alternatively) optimize in the observation latent space, following
\begin{equation}\label{eq:latentcomp_zy}
\hat \bfz_y \in \argmin_{\bfz_y} \left\| d_y(\bfz_y)- \bfy \right\| \qquad \text{with} \quad \hat \bfx = (d_x\circ m^{\gets}) (\hat \bfz_y).
\end{equation}

\noindent To address missing or corrupted data, where we observe $\bfy_\text{sub} = P(\bfy)$ with projection $P$, we reformulate \Cref{eq:latentcomp_zx} as 
\begin{equation}\label{eq:latentcomp_zx_P}
    \hat \bfz_x\in \argmin_{\bfz_x} \left\| (P \circ d_y \circ m^{\to})(\bfz_x)- \bfy_\text{sub} \right\| \qquad\text{ with } \quad \hat \bfx = d_x(\hat \bfz_x),
\end{equation}
or analogously, \Cref{eq:latentcomp_zy} as 
\begin{equation}\label{eq:latentcomp_zy_P}
\hat \bfz_y \in \argmin_{\bfz_y} \left\| (P\circ d_y)(\bfz_y)- \bfy_\text{sub} \right\| \qquad \text{with} \quad \hat \bfx = (d_x\circ m^{\gets}) (\hat \bfz_y).
\end{equation}
Note, the above problems remain ill-posed, and the model reconstruction quality depends on the specific inverse problem, the training dataset quality, and the designs for the encoder, the decoder, and $m^\gets$. 
To address this ill-posedness, regularization can be employed in \Cref{eq:latentcomp_zx_P} and \Cref{eq:latentcomp_zy_P}, where a regularization functional encoding prior knowledge on $\bfx$ could be used on $d_x(\bfz_x)$ or $d_x(m^{\gets} \bfz_y)$ to enforce prior knowledge on the LSI solution $\widehat\bfx$.  In our numerical examples, we find that such regularization is unnecessary, and so for simplicity, we exclude it.

The numerical and computational considerations for this optimization problem are straightforward. The latent map $m^{\gets}$ is typically a simple linear operator and will not significantly contribute to the computational cost. Backpropagation through the decoder network $d_y$ and potentially $m^{\gets}$ if using \Cref{eq:latentcomp_zx} is readily available. For simplicity, we utilize formulation \Cref{eq:latentcomp_zy}. We note that empirically, optimizing in $\calZ_\calY$ provides better reconstructions of $\bfx$ than reconstructions obtained by optimizing in $\calZ_\calX$.  We aim to explain this empirical observation with theoretical results for linear constructions in \Cref{sub:linear}.   

\bigskip

\subsection{Latent Space Inference with Linear Constructions}
\label{sub:linear}
Let us consider latent space inference where the forward model $A$, the projection $P$, the partially observed $\bfy_\text{sub}$, the encoders $e_x$, $e_y$, the decoders $d_x$, $d_y$, and the latent space mappings $m^{\to}$ and $m^{\gets}$ are all linear, denoted $\bfA$, $\bfP$, $\bfE_x$, $\bfE_y$, $\bfD_x$, $\bfD_y$, $\bfM^{\to}$ and $\bfM^{\gets}$ respectively.  Let $X$ be a random variable with realizations $\bfx \in \bbR^n$, finite first moment $\bbE X = \bfmu_x$, and symmetric positive definite (SPD) second moment $\bbE XX\t=\bfGamma_x=\bfL_x\bfL_x\t$.  We let $\epsilon$ be a random variable independent from $X$ with realizations in $\bbR^q$, mean zero $\bbE \epsilon = 0$, and SPD second moment $\bfGamma_\epsilon=\bfL_\epsilon \bfL_\epsilon\t$.  From \Cref{eq:ip}, we define the random variable $Y$ with realizations $\bfy\in\bbR^q$, $\bbE Y = \bfA \bfmu_x$, and SPD second moment $\bbE YY\t = \bfGamma_y = \bfA \bfGamma_x \bfA\t + \bfGamma_\epsilon$.  We denote the singular value decomposition (SVD) of a matrix $\bfW\in\bbR^{q\times n}$ as \begin{equation*}
    \bfW = \bfU_W \bfSigma_W \bfV_W\t,
\end{equation*}
where $\bfSigma_W\in\bbR^{q\times n}$ is a diagonal matrix with the singular values $\sigma_1\geq\sigma_2\geq...\geq\sigma_{\min(q,n)}$ on its diagonal, and $\bfU_W=[\bfu_1, \bfu_2, \ldots, \bfu_q]\in\bbR^{q\times q}$ and $\bfV_W=[\bfv_1, \bfv_2, \ldots, \bfv_n]\in\bbR^{n\times n}$ are orthogonal matrices whose columns are the left and right singular vectors of $\bfW$, respectively.  We denote the truncated SVD of $\bfW$, for positive integer $r\leq\min{(q,n)}$ as 
\begin{equation*}
    \bfW_r = \bfU_{W,r} \bfSigma_{W,r} \bfV_{W,r}\t = \begin{bmatrix}
        \vertbar & \vertbar & & \vertbar \\
        \bfu_1 & \bfu_2 & \cdots & \bfu_r \\
        \vertbar & \vertbar& & \vertbar \\  
    \end{bmatrix}
    \begin{bmatrix}
        \sigma_1 & & & \\
         & \sigma_2 & & \\
         & & \ddots & \\
         & & & \sigma_r
    \end{bmatrix}
    \begin{bmatrix}
        \horzbar & \bfv_1\t & \horzbar \\
        \horzbar & \bfv_2\t & \horzbar \\
        & \vdots & \\
        \horzbar & \bfv_r\t & \horzbar \\
    \end{bmatrix}.
\end{equation*}

From \cite{hart2025paired}, in a Bayes risk minimization sense, optimal paired linear encoders, decoders, and latent space mappings are given by
\begin{align*}
    \bfE_x &= \bfU_{L_x, \ell_x}\t, \qquad\quad &\bfD_x = \bfU_{L_x, \ell_x},\quad \quad  \qquad\qquad\qquad &\bfM^{\to} = \bfU_{L_y, \ell_y}\t \bfA \bfU_{L_x, \ell_x}, \\
    \bfE_y &= \bfU_{L_y \ell_y}\t, \qquad\quad &\bfD_y = \bfU_{L_y, \ell_y}, \quad \text{and} \qquad\qquad\qquad &\bfM^{\gets} = \bfSigma_{L_x, \ell_x}^2\bfU_{L_x, \ell_x}\t \bfA\t \bfU_{L_y, \ell_y}\bfSigma_{L_y, \ell_y}^{-2}.
\end{align*}
In this linear case, the optimal latent representations in $\calZ_\calX$ and $\calZ_\calY$ from \Cref{eq:latentcomp_zx_P} and \Cref{eq:latentcomp_zy_P} can be formulated as
\begin{equation}
    \hat \bfz_x \in \argmin_{\bfz_x} \|\bfP\bfD_y\bfM^{\to} \bfz_x - \bfP\bfy\|_2^2 \qquad \text{ and } \qquad
    \hat \bfz_y \in \argmin_{\bfz_y} \|\bfP\bfD_y \bfz_y - \bfP\bfy\|_2^2
\end{equation}
where the reconstruction of $\bfx$ from optimization in $\calZ_\calX$ is $\hat\bfx_{z_x} = \bfD_x \bfz_x$ and the reconstruction of $\bfx$ from optimization in $\calZ_\calY$ is $\hat\bfx_{z_y} = \bfD_x \bfM^{\gets} \bfz_y$.  This constitutes a linear least squares problem; using $\bfW^\dagger$ to denote the Moore-Penrose pseudoinverse of $\bfW$, closed-form solutions for $\bfz_x$ and $\bfz_y$ are given by
\begin{equation}
    \hat \bfz_x = (\bfP \bfD_y \bfM^{\to})^\dagger \bfP\bfy \qquad \text{ and } \qquad \hat \bfz_y = (\bfP\bfD_y)^\dagger\bfP\bfy.
\end{equation}

Substituting optimal choices for the decoders and latent mappings in a Bayes risk minimization sense, we find
\begin{equation*}
    \hat\bfx_{z_x} = \bfU_{L_x, \ell_x}(\bfP \bfU_{L_y, \ell_y} \bfU_{L_y, \ell_y}\t \bfA \bfU_{L_x, \ell_x})^\dagger \bfP\bfy
\end{equation*}
and
\begin{equation*}
    \hat \bfx_{z_y} = \bfU_{L_x, \ell_x} \bfSigma_{L_x, \ell_x}^2\bfU_{L_x, \ell_x}\t \bfA\t \bfU_{L_y, \ell_y}\bfSigma_{L_y, \ell_y}^{-2}(\bfP \bfU_{L_y, \ell_y})^\dagger \bfP\bfy.
\end{equation*}
Note that when $\bfP$ is the identity matrix $\bfI\in\bbR^{q\times q}$, and $\ell_x = n$ and $\ell_y=q$, then $\bfU_{L_x, \ell_x}=\bfU_{L_x}$ and $\bfU_{L_y, \ell_y}=\bfU_{L_y}$ are invertible, and 
\begin{align*}
    \hat\bfx_{z_x} &= \bfU_{L_x}(\bfU_{L_y} \bfU_{L_y}\t \bfA \bfU_{L_x})^\dagger \bfy = \bfU_{L_x} \bfU_{L_x}\t \bfA^\dagger \bfU_{L_y} \bfU_{L_y}\t \bfy\\
    \hat\bfx_{z_y} &= \bfU_{L_x} \bfSigma_{L_x}^2\bfU_{L_x}\t \bfA\t \bfU_{L_y}\bfSigma_{L_y}^{-2}\bfU_{L_y}\t\bfy = \bfGamma_x \bfA\t \bfGamma_y^{-1} \bfy.    
\end{align*}
These results are consistent with \cite{hart2025paired}, and we note that the optimal solutions found in $\calZ_\calX$ and $\calZ_\calY$ are inherently different.  Specifically, the solution found in $\calZ_\calY$ includes the diagonal matrices $\bfSigma_{L_x}^2$ and $\bfSigma_{L_y}^{-2}$ with the singular values of $\bfGamma_{x}$ and $\bfGamma_{y}$, whereas the solution found in $\calZ_\calX$ does not.  Empirically, we observed that in the nonlinear setting, the solutions found by optimizing in $\calZ_\calY$ were better than those found in $\calZ_\calX$, indicating that this scaling may be critical.  Thus, we focus on latent space inference in $\calZ_\calY$.

\subsection{Residual and Error Bounds for Latent Space Inference}

Assume that linear or nonlinear paired autoencoders are trained, parameterizing $e_x$, $d_x$, $e_y$, $d_y$, $m^\gets$, and $m^\to$.  Then with additional assumptions on Lipschitz continuity and on the consistency of PAIR components, we can bound the residual $\|P(A(\widehat\bfx)) - \bfy_{\rm sub}\|$ and the reconstruction error $\|\widehat\bfx-\bfx\|$, where $\widehat\bfx$ is the LSI solution obtained via \Cref{eq:latentcomp_zy_P}.  Beginning with a simple projection bound in \Cref{fact:proj_res}, we successively build toward a general error estimate for $\|\widehat\bfx - \bfx\|$ in Proposition \ref{proposition}, and finally relate this reconstruction error to the forward-model residual in the observation space in Statement \ref{fact:actual_proj_res} to show how discrepancies between reconstructed and true signals propagate through the learned mappings.

\begin{fact}\label{fact:proj_res}
    Let $P:\calY\to\calY_{\rm sub}$ be Lipschitz continuous on the metric spaces $(\calY, \|\cdot\|)$ and $(\calY_{\rm sub}, \|\cdot\|)$ with Lipschitz constant $L_P\geq0$.  Also assume that there exists $\bfz=e_y(\bfy)\in\calZ_\calY$, and that $\|(d_y \circ e_y)(\bfw)-\bfw\|\leq\varepsilon_y$ for all $\bfw\in\calY$.  Then, $\| P(d_y(\widehat\bfz))-P(\bfy) \| \leq L_p \varepsilon_y$. 
\end{fact}
\begin{proof}
    By triangular inequality,
    \begin{align*}
        \|P(d_y(\widehat\bfz)) - P(\bfy)\| &\leq \|P(d_y(\bfz)) - P(\bfy)\| \leq L_p \|(d_y\circ e_y)(\bfy) - \bfy \| \leq L_p \varepsilon_y.
    \end{align*}
\end{proof}

\noindent From LSI, $\widehat\bfz$ minimizes $\|P(d_y(\widehat\bfz)) - P(\bfy)\|$, and the performance of the data autoencoder $d_y\circ e_y$ bounds the magnitude of this difference.

\begin{prop}\label{proposition}
Let all the assumptions of Statement \ref{fact:proj_res} hold. Let $\mathcal S\subseteq \calY$ be a subset such that $\bfy\in\mathcal S$ and $d_y(\widehat\bfz)\in\mathcal S$, and assume that $P:\calY\to\calY_{\rm sub}$ satisfies a restricted bi-Lipschitz condition on $\mathcal S$, i.e., there exist constants
\[ 0<\alpha_P \le \beta_P < \infty \]
such that
\begin{equation}\label{eq:restricted_bilip}
    \alpha_P \|\bfw_1 - \bfw_2\| \leq \|P(\bfw_1) - P(\bfw_2)\| \leq \beta_P\|\bfw_1 - \bfw_2\|
\end{equation}
for all $\bfw_1, \bfw_2 \in \mathcal S$. Additionally, let $d_x$, $m^\gets$, and $e_y$ be Lipschitz continuous on their corresponding metric spaces with Lipschitz constants $L_{d_x}$, $L_{m^\gets}$, and $L_{e_y}$, respectively. Assume that $\|(d_x\circ e_x)(\bft) - \bft\|\leq\varepsilon_x$ for all $\bft \in \calX$, $\|(m^\gets\circ m^\to)(\bft)-\bft\|\leq\gamma_m$ for all $\bft\in\calZ_\calX$, $\|\bfy-(d_y\circ m^\to \circ e_x)(\bfx)\|\leq \delta$ for all $\bfy = A(\bfx)+\bfepsilon$, and that
\begin{equation}\label{eq:autorepre}
\|\bfv_1 - \bfv_2\| \leq \|(e_y\circ d_y)(\bfv_1)-(e_y\circ d_y)(\bfv_2)\| \qquad \text{for all } \quad \bfv_1, \bfv_2 \in \calZ_\calY.
\end{equation}
Then,
\begin{equation*}
\|\widehat\bfx - \bfx\|
\;\le\;
L_{d_x}\left(
L_{m^\gets}L_{e_y}\left(\frac{\beta_P}{\alpha_P}\,\varepsilon_y + \delta\right)
+\gamma_m\right)
+\varepsilon_x.
\end{equation*}
\end{prop}

\begin{proof}
With the definition of $\widehat\bfx$ and the triangle inequality,
\begin{align*}
\|\widehat\bfx - \bfx\|
&\leq \|(d_x \circ m^\gets)(\widehat\bfz) - (d_x\circ e_x)(\bfx)\| + \|(d_x\circ e_x)(\bfx) - \bfx\| \\
&\leq L_{d_x}\, \|m^\gets(\widehat\bfz) - e_x(\bfx)\| + \varepsilon_x.
\end{align*}
Further,
\begin{align*}
\|m^\gets(\widehat\bfz) - e_x(\bfx)\|
&\leq \|m^\gets(\widehat\bfz) - (m^\gets\circ m^\to\circ e_x)(\bfx)\| 
      + \|(m^\gets\circ m^\to)(e_x(\bfx)) - e_x(\bfx)\| \\
&\leq L_{m^\gets}\,\|\widehat\bfz - (m^\to\circ e_x)(\bfx)\| + \gamma_m.
\end{align*}
For the latent discrepancy, using \Cref{eq:autorepre} and Lipschitz continuity of $e_y$,
\begin{align*}
\|\widehat\bfz - (m^\to\circ e_x)(\bfx)\|
&\leq \|(e_y\circ d_y)(\widehat\bfz) - (e_y\circ d_y\circ m^\to\circ e_x)(\bfx)\| \\
&\leq L_{e_y}\,\|d_y(\widehat\bfz) - (d_y\circ m^\to\circ e_x)(\bfx)\| \\
&\leq L_{e_y}\Big(\|d_y(\widehat\bfz) -\bfy \| + \| \bfy - (d_y\circ m^\to\circ e_x)(\bfx)\|\Big) \\
&\leq L_{e_y}\big(\|d_y(\widehat\bfz) -\bfy \| + \delta\big).
\end{align*}
Since $\bfy,d_y(\widehat\bfz)\in\mathcal S$, the restricted lower bound in \Cref{eq:restricted_bilip} yields
\[
\|d_y(\widehat\bfz)-\bfy\|
\leq \frac{1}{\alpha_P}\,\|P(d_y(\widehat\bfz)) - P(\bfy)\|.
\]
Combining this with the upper bound in \Cref{eq:restricted_bilip} and Statement \ref{fact:proj_res}, we obtain
\begin{align*}
\|d_y(\widehat\bfz)-\bfy\|
&\leq \frac{1}{\alpha_P}\|P(d_y(\widehat\bfz)) - P(\bfy)\|
\leq \frac{\beta_P}{\alpha_P}\varepsilon_y.
\end{align*}
Plugging into the previous estimates gives
\begin{align*}
\|\widehat\bfx - \bfx\|
&\leq L_{d_x}\Big(
L_{m^\gets}L_{e_y}(\tfrac{\beta_P}{\alpha_P}\varepsilon_y+\delta)+\gamma_m\Big)+\varepsilon_x.
\end{align*}
\end{proof}

The stability bound in Proposition \ref{proposition} relies on a restricted bi-Lipschitz property of the partial observation operator $P$ on the subset $\mathcal S$ of relevant signals, rather than on the entire space $\calY$. For masking and subsampling operators, such a restriction is unavoidable, since, in applications, $P$ is non-injective and admits a nontrivial nullspace. The lower constant $\alpha_P$ characterizes how well $P$ separates distinct elements on $\mathcal S$; when $\alpha_P$ becomes small due to severe data loss or alignment of the learned data manifold with $\ker{P}$, the inverse problem becomes ill-posed on $\mathcal S$ and the bound necessarily degrades (see Experiment in \Cref{sec:seismic},  \Cref{fig:LSI_metrics_seismic}).
In practice, the diagnostic quantities provided by the PAIR framework offer an empirical indication of whether a given observation lies within this stable regime, thus identifying scenarios where latent space inference may be needed.

For tight bounds, we assume that (i) the errors for both the data and model autoencoders are small, (ii) $(d_y\circ m^\to \circ e_x)(\cdot)$ is a good surrogate model for $A$, and (iii) the mappings between latent spaces approximately undo each other.  A similar bound can be obtained for the residual norm.

\begin{fact}\label{fact:actual_proj_res}
    Let all the assumptions of Proposition \ref{proposition} hold, and additionally assume that $A:\calX\to\calY$ is Lipschitz continuous on $(\calX, \|\cdot\|)$ and $(\calY, \|\cdot\|)$ with Lipschitz constant $L_A\geq0$.  Then, $$\|P(A(\widehat\bfx))-\bfy_{\rm sub}\|\leq \beta_P L_A\Big( L_{d_x} \left(L_{m^\gets}L_{e_y} ( \tfrac{\beta_P}{\alpha_P} \varepsilon_y + \delta) +\gamma_m\right) + \varepsilon_x \Big) + \beta_P \|\bfepsilon\|.$$
\end{fact}
\begin{proof}
    \begin{align*}
        \|P(A(\widehat\bfx))-\bfy_{\rm sub}\| &\leq \beta_P \|A(\widehat\bfx)-(A(\bfx)+\bfepsilon)\| \\
        & \leq \beta_P L_A \|\widehat\bfx - \bfx\| + \beta_P\|\bfepsilon\| \\
        & \leq \beta_P L_A\left( L_{d_x} \left(L_{m^\gets}L_{e_y} (\tfrac{\beta_P}{\alpha_P} \varepsilon_y + \delta) +\gamma_m\right) + \varepsilon_x \right) +\beta_P \|\bfepsilon\|.
    \end{align*}
\end{proof}

These results help to make explicit how each component of the paired autoencoder framework contributes to the total LSI error and residual norm.  The bound in Proposition \ref{proposition} shows that reconstruction accuracy in $\calX$ is controlled by the Lipschitz constants of the mappings and by the internal reconstruction errors $\varepsilon_x$, $\varepsilon_y$, $\gamma_m$, and $\delta$. Statement \ref{fact:actual_proj_res} extends this reasoning to the data space, demonstrating that the projected residual $\|P(A(\widehat\bfx))-\bfy_{\rm sub}\|$ grows at most linearly with these same quantities. Together, these results quantify the propagation of approximation errors through LSI and provide sufficient conditions under which the inversion remains stable and well-behaved.

\section{Numerical Results and Comparisons}
\label{sec:numerics}
In this section, we describe two example applications and demonstrate latent space inference under various scenarios of observational inconsistencies. We consider examples in computed tomography (CT) image reconstruction and in cross-well seismic inversion.  
For both examples, we consider latent space inference with PAIR by solving  \Cref{eq:latentcomp_zy_P}, and we denote our new approach `PAIR + LSI'.  We provide comparisons to paired autoencoders without latent-space optimization, which was previously shown to be superior to end-to-end inversion approaches in \cite{chung2024paired, hart2025paired}, even without observational inconsistencies.  
In the first example, we discuss metrics for out-of-distribution detection, and in the second example, we compare PAIR and PAIR + LSI for different percentages of missing data.

\subsection{Example 1: CT Imaging}

CT image reconstruction involves reconstructing an internal image of an object from X-ray measurements taken at multiple angles outside of the object.  These measurements form a sinogram, which mathematically corresponds to the Radon transform of the object.  Recovering the original anatomy requires inverting this transform, which is inherently an ill-posed problem \cite{natterer2001mathematics}.  Additionally, variability in scanner configurations across machines and institutions introduces differences in data sizes and types, which any data-driven method must account for to be widely applicable.

For this example, we use the 2DeteCT dataset \cite{kiss20232detect} to evaluate our latent space inference methods in addressing CT tasks with incomplete data.  From this dataset, we use downsampled reconstructions of the phantom ``anatomy" of nuts, dried fruits, stones, and filler materials, imaged to produce training samples that share features and contrasts similar to abdominal CT scans.  We simulate the forward CT process using `PRTomo' from the IRTools package \cite{gazzola2019ir}, and added $10\%$ white noise to each sinogram.

We train the two autoencoders $a_y = (d_y\circ e_y)$ and $a_x=(d_x\circ e_x)$ and latent space mappings $m^{\to}$ and $m^{\gets}$ together, minimizing the combined loss in \Cref{eq:PAIRloss}
for $\{\bfY, \bfX\}$ in the set of training samples $\{\bfY^{\rm train}_i, \bfX^{\rm train}_i\}_{i=1}^{\num{4900}}$, where $\bfY^{\rm train}_i\in\bbR^{96\times181}$ and $\bfX^{\rm train}_i\in\bbR^{52\times52}$.  The encoder $e_x$ consists of an initial convolutional layer, followed by three residual blocks, each made of two $3\times3$ convolutions with batch normalization and ReLU activation and downsampling by stride-2 convolution at the start of each block.  After resizing to 512 and flattening, this results in a more compact latent representation $\bfz_x=e_x(\bfX)$.  The decoder $d_x$ mirrors this process.  Starting from $\bfz_x$, it applies a fully connected linear layer to reshape to the final encoder feature size before flattening, then applies three residual blocks with upsampling (via transposed convolutions of stride 2), halving the channel dimension and doubling the spatial resolution at each stage until the original input dimension is recovered.  The encoder $e_y$ and decoder $d_y$ are similar, but each includes 6 residual blocks, rather than 3.  Both $m^{\to}$ and $m^{\gets}$ are $512\times512$ fully connected linear layers.  Training was conducted using an Adam optimizer with a learning rate of $0.001$ for 100 epochs.

For the CT example, we simulated missing data in two ways: (1) randomly missing angles, representing scenarios where measurements at certain projection angles are corrupted or unavailable; and (2) contiguous blocks of missing angles, modeling situations such as limited-angle tomography, where a range of projection angles cannot be measured. These scenarios can be represented using a projection operation (or masking matrix) that zeros out or removes columns of the full-data sinogram $\bfY$.  Let $P_{\rm rand}$ be a projection that zeros a fixed collection of 45 randomly selected columns from $\bfY$, and let $P_{\rm rect}$ be a projection that zeros a fixed block of 45 consecutive columns of $\bfY$.  Note that these are two different ways of obscuring $25\%$ of the original set of projection angles.  See the top row of \Cref{fig:2detect_ex0} for an illustration of the sinograms with and without missing data.

For the test set $\{\bfY^{\rm test}_i, \bfX^{\rm test}_i\}_{i=1}^{100}$ (not used during training), we consider three settings: reconstructing $\bfX^{\rm test}$ from full sinograms $\bfY^{\rm test}$; from sinograms with randomly missing angles, $P_{\rm rand}(\bfY^{\rm test})$; and from sinograms with a block of missing angles $P_{\rm rect}(\bfY^{\rm test})$.
For each of these datasets, we compute reconstructions using PAIR+LSI
following \Cref{eq:latentcomp_zy_P}.  Starting from an initial guess of $\bfz_0=e_y(\bfY_{\rm sub})$, optimization was conducted with an L-BFGS optimizer for ten iterations with Strong-Wolfe line search to ensure stable updating.
We compare the reconstruction performance to using PAIR directly (e.g., \Cref{eq:PAIRinverse}) and to end-to-end encoder-decoders. The end-to-end encoder-decoders have exactly the same architecture as the PAIR network's end-to-end prediction from \Cref{eq:PAIRinverse}, but the networks are parameterized by minimizing the loss
\begin{equation*}\label{eq:EDloss}
    \sum_{\bfY,\bfX} \|(d \circ m \circ e)(\bfY)-\bfX\|_{\rm F}^2
\end{equation*}
for $\{\bfY, \bfX\}$ in the three sets of training samples $\{\bfY^{\rm train}_i, \bfX^{\rm train}_i\}_{i=1}^{\num{4900}}$, in $\{P_{\rm rand}(\bfY^{\rm train}_i), \bfX^{\rm train}_i\}_{i=1}^{\num{4900}}$, and in $\{P_{\rm rect}(\bfY^{\rm train}_i), \bfX^{\rm train}_i\}_{i=1}^{\num{4900}}$ for the $Y \to X$ encoder-decoder, $P_{\rm rand}(Y) \to X$ encoder-decoder, and $P_{\rm rect}(Y) \to X$ encoder-decoder, respectively.
Training for each encoder-decoder was conducted with an Adam optimizer with learning rate $0.001$ for 100 epochs, as with the PAIR network.

 Moreover, we assess performance on an out-of-distribution dataset $\{\bfX_i^{\rm OOD},\bfY^{\rm OOD}_i\}_{i=1}^{125}$ that was provided in \cite{kiss20232detect}, where the observations were obtained using a different trail mix combination. We reconstruct $\bfX^{\rm OOD}$ from $\bfY^{\rm OOD}$, $P_{\rm rand}(\bfY^{\rm OOD})$, and $P_{\rm rect}(\bfY^{\rm OOD})$.  Lastly, we repeat experiments with $\{\bfY^{\rm test}_i, \bfX^{\rm test}_i\}_{i=1}^{100}$ and $\{\bfX_i^{\rm OOD},\bfY^{\rm OOD}_i\}_{i=1}^{125}$ reconstructing $\bfX^{\rm test}$ from $\widetilde P_{\rm rand}(\bfY^{\rm test})$ and $\widetilde P_{\rm rect}(\bfY^{\rm test})$, and $\bfX^{\rm OOD}$ from $\widetilde P_{\rm rand}(\bfY^{\rm OOD})$ and $\widetilde P_{\rm rect}(\bfY^{\rm OOD})$, where $\widetilde P_{\rm rand}$ zeros a \textit{different} fixed collection of 45 randomly selected columns than $P_{\rm rand}$ and $\widetilde P_{\rm rect}$ zeros a \textit{different} fixed block of 45 consecutive columns than $P_{\rm rect}$.

\begin{table}[ht]
\centering
\renewcommand{\arraystretch}{1.5}
\footnotesize

\begin{subtable}{\textwidth}
    \centering
    \begin{tabular}{|l|cc|cc|}
    \hline
         \multirow{2}{*}{Method} & \multicolumn{2}{c|}{$\bfY$} & \multicolumn{2}{c|}{$\bfY^{\rm OOD}$} \\
         & RRE $\downarrow$ & SSIM $\uparrow$ & RRE $\downarrow$ & SSIM $\uparrow$ \\
         \hline
         PAIR & \textbf{0.224} & \textbf{0.894} & \textbf{0.195} & \textbf{0.912} \\
         $\bfY \to \bfX$ Encoder-Decoder & 0.291 & 0.850 & 0.256 & 0.867 \\
         PAIR + LSI & 0.270 & 0.844 & 0.266 & 0.856 \\
    \hline
    \end{tabular}
    \caption{\textit{Results from full observational data.
    }}
    \label{tab:results_full_a}
\end{subtable}

\vspace{0.5cm}

\begin{subtable}{\textwidth}
    \centering
    \renewcommand{\arraystretch}{1.75}
    \begin{tabular}{|l|cc|cc|cc|cc|}
    \hline
         \multirow{2}{*}{Method} & \multicolumn{2}{c|}{$P_{\rm rand}(\bfY)$} & \multicolumn{2}{c|}{$P_{\rm rand}\left(\bfY^{\rm OOD}\right)$} & \multicolumn{2}{c|}{$\widetilde P_{\rm rand}(\bfY)$} & \multicolumn{2}{c|}{$\widetilde P_{\rm rand}\left(\bfY^{\rm OOD}\right)$} \\
         & RRE $\downarrow$ & SSIM $\uparrow$ & RRE $\downarrow$ & SSIM $\uparrow$& RRE $\downarrow$ & SSIM $\uparrow$ & RRE $\downarrow$ & SSIM $\uparrow$ \\
         \hline
         PAIR & 0.334 & 0.787 & 0.356 & 0.766 & 0.330 & 0.788 & 0.351 & 0.769 \\
         $P_{\rm rand}(\bfY) \to \bfX$ Encoder-Decoder & 0.345 & 0.819 & 0.314 & 0.820 & 0.335 & 0.826 & 0.297 & 0.835 \\
         PAIR + LSI & \textbf{0.274} & \textbf{0.841} & \textbf{0.271} & \textbf{0.853} & \textbf{0.275} & \textbf{0.840} & \textbf{0.271} & \textbf{0.853} \\
    \hline
    \end{tabular}
    \caption{\textit{
    Results from observational data with randomly missing angles.
    }}
    \label{tab:results_full_b}
\end{subtable}

\vspace{0.5cm}

\begin{subtable}{\textwidth}
    \centering
    \renewcommand{\arraystretch}{1.75}
    \begin{tabular}{|l|cc|cc|cc|cc|}
    \hline
         \multirow{2}{*}{Method} & \multicolumn{2}{c|}{$P_{\rm rect}(\bfY)$} & \multicolumn{2}{c|}{$P_{\rm rect}(\bfY_{\rm OOD})$} & \multicolumn{2}{c|}{$\widetilde P_{\rm rect}(\bfY)$} & \multicolumn{2}{c|}{$\widetilde P_{\rm rect}(\bfY_{\rm OOD})$} \\
         & RRE $\downarrow$ & SSIM $\uparrow$ & RRE $\downarrow$ & SSIM $\uparrow$& RRE $\downarrow$ & SSIM $\uparrow$ & RRE $\downarrow$ & SSIM $\uparrow$ \\
         \hline
         PAIR & 0.862 & 0.427 & 1.043 & 0.398 & 0.9333 & 0.414 & 1.124 & 0.392 \\
         $P_{\rm rect}(\bfY) \to \bfX$ Encoder-Decoder & \textbf{0.266} & \textbf{0.862} & \textbf{0.233} & \textbf{0.881} & 0.447 & 0.693 & 0.436 & 0.681 \\
         PAIR + LSI & 0.297 & 0.823 & 0.293 & 0.832 & \textbf{0.303} & \textbf{0.816} & \textbf{0.286} & \textbf{0.835} \\
    \hline
    \end{tabular}
    \caption{\textit{Results from observational data with a block of missing angles.}}
    \label{tab:results_full_c}
\end{subtable}

\caption{Comparison of relative reconstruction errors (RRE) and SSIM values, averaged over the test set, under different experimental settings.  The header indicates what $\bfX$ is reconstructed from.  Each subtable corresponds to a different masking scheme: (a) none (e.g., full data), (b) random, and (c) block. `OOD' denotes out-of-distribution. The smallest values of average RRE and the largest values of average SSIM are bold-faced.}
\label{tab:results_all}
\end{table}

In \Cref{tab:results_all}, we provide the relative reconstruction errors, calculated as $\|\bfX_{\rm pred}- \bfX\|_{\rm F} / \|\bfX\|_{\rm F}$, where $\bfX$ is the true image and $ \bfX_{\rm pred}$ is the predicted image, averaged over the test set, for various examples.  We also include the Structural Similarity Index Measure (SSIM) to give an idea of another measure of perceived quality \cite{wang2004image}. From \Cref{tab:results_full_a}, we find that when full angles are available, the PAIR method alone outperforms the end-to-end encoder-decoder network and PAIR $+$ LSI, on both in- and out-of-distribution samples. However, the benefit of PAIR $+$ LSI is evident in the missing data regime. From \Cref{tab:results_full_b}, we see that with randomly blocked columns of $\bfY$, PAIR $+$ LSI performs the best in all cases, when $\bfY$ is in and out of distribution and for both masking projections $P_{\rm rand}$ and $\widetilde P_{\rm rand}$.  We note that the $P_{\rm rand}(\bfY) \to \bfX$ encoder-decoder also performs similarly for the data masked with $P_{\rm rand}$, with which it was trained, and with $\widetilde P_{\rm rand}$, which shares a similar structure, but corresponds to a different mask.
This is not true with the $P_{\rm rect}(\bfY) \to \bfX$ encoder-decoder, which performs significantly worse when reconstructing $\bfX$ from $\widetilde P_{\rm rect}(\bfY)$ compared to $\bfX$ from $ P_{\rm rect}(\bfY)$. From \Cref{tab:results_full_b}, we observe that PAIR $+$ LSI performs similarly on the two tasks, as we expect, and better than the encoder-decoder on reconstructing $\bfX$ from $\widetilde P_{\rm rect}(\bfY)$. \Cref{fig:2detect_ex0} shows example reconstructions from each of the methods for an in-distribution test example, while \Cref{fig:2detect_ex1_OOD} shows reconstructions of an out-of-distribution test sample. Despite being classified as out-of-distribution in the context of the CT application (e.g., using a different trail mix collection), the machine learning model successfully reconstructs these inputs with high accuracy.

\begin{figure}
    \centering
    \begin{tabular}{lcccccc}
        & $\bfY$ & $P_{\rm rand}(\bfY)$ & $P_{\rm rect}(\bfY)$ & $\widetilde P_{\rm rand}(\bfY)$ & $\widetilde P_{\rm rect}(\bfY)$ \\
         & \includegraphics[width=0.15\linewidth]{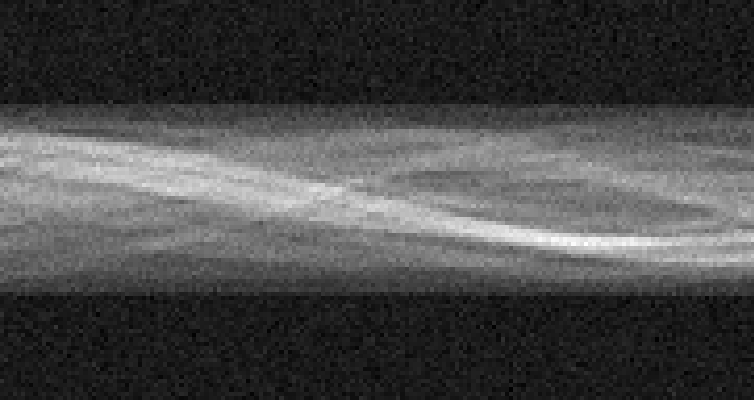}
         & \includegraphics[width=0.15\linewidth]{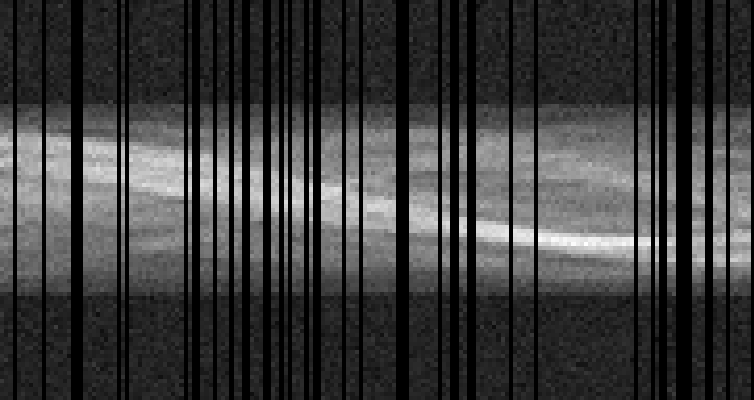}
         & \includegraphics[width=0.15\linewidth]{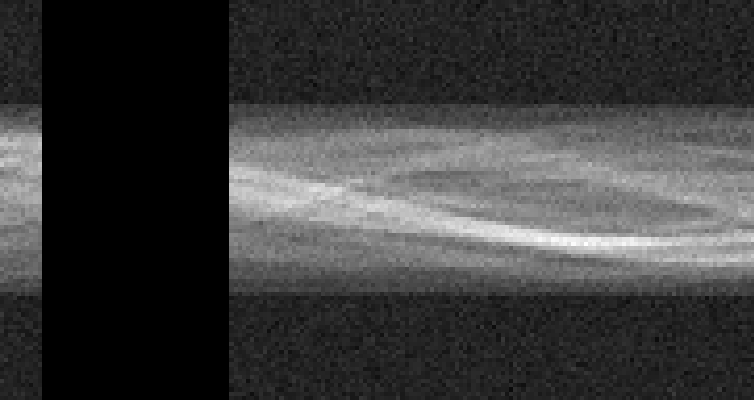}
         & \includegraphics[width=0.15\linewidth]{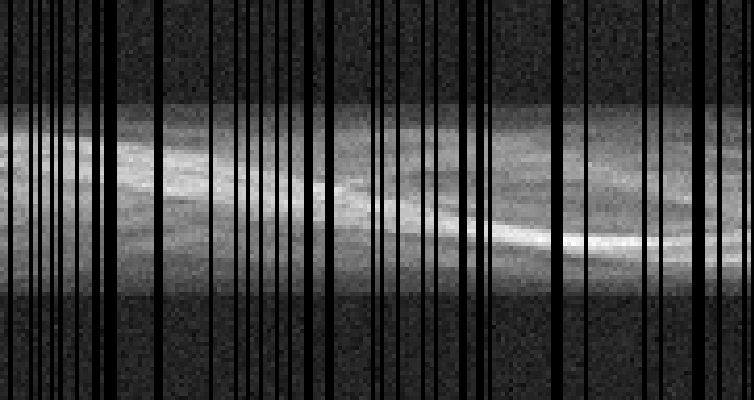}
         & \includegraphics[width=0.15\linewidth]{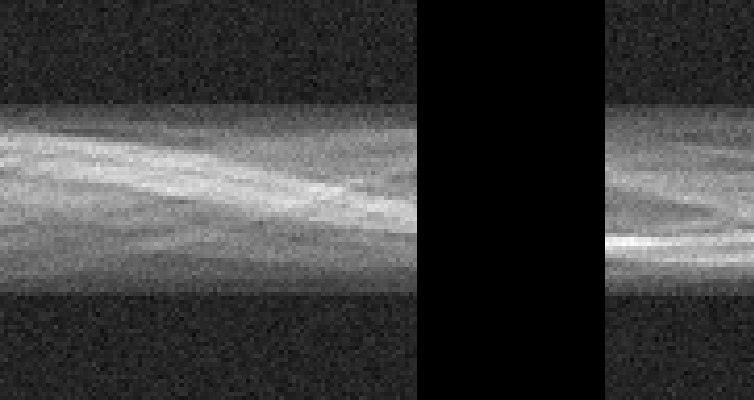}
         \\
         & \textcolor{white}{.} & & & & & \\
         \hline\\
         & \textcolor{white}{.} & & & & & \\
         \rotatebox{90}{$\qquad$ PAIR}
         & \includegraphics[width=0.15\linewidth]{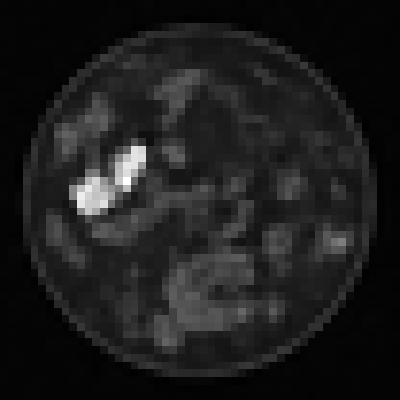}
         & \includegraphics[width=0.15\linewidth]{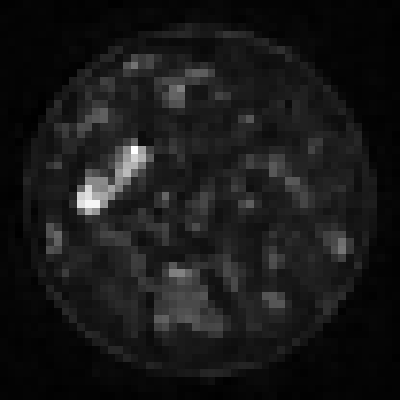}
         & \includegraphics[width=0.15\linewidth]{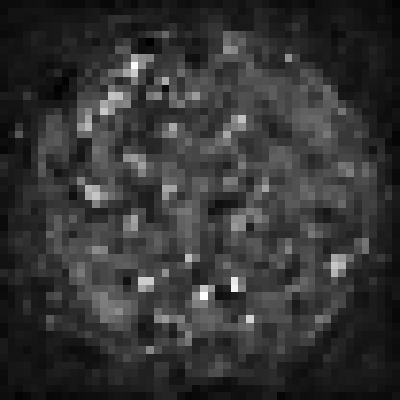}
         & \includegraphics[width=0.15\linewidth]{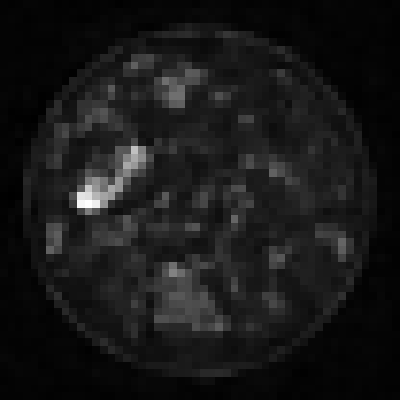}
         & \includegraphics[width=0.15\linewidth]{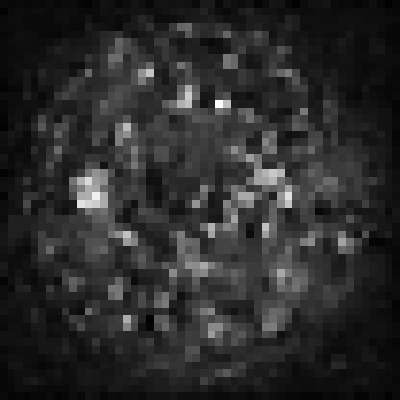}
         \\
         \rotatebox{90}{$\,\,$ En/De-coders}
         & \includegraphics[width=0.15\linewidth]{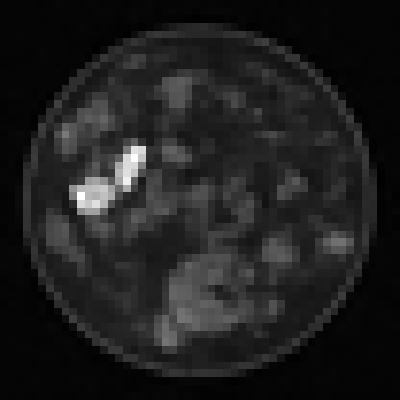}
         & \includegraphics[width=0.15\linewidth]{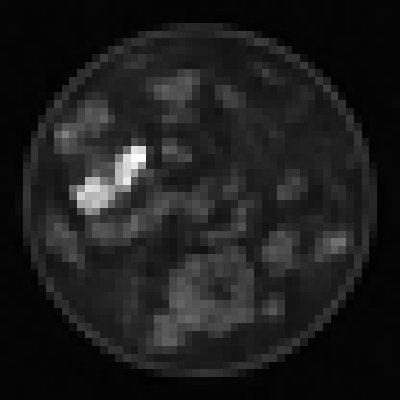}
         & \includegraphics[width=0.15\linewidth]{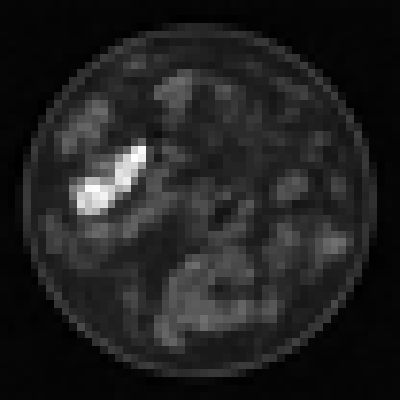}
         & \includegraphics[width=0.15\linewidth]{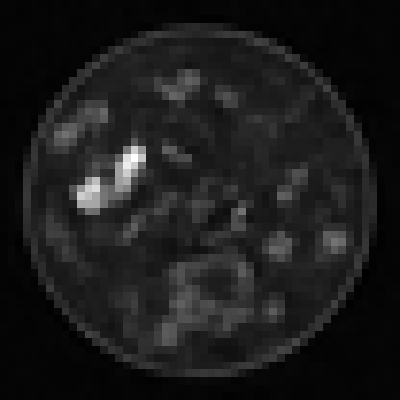}
         & \includegraphics[width=0.15\linewidth]{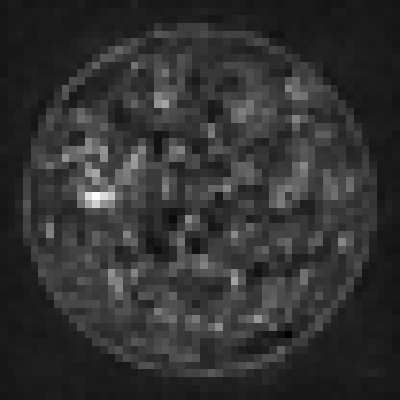}
         \\
         \rotatebox{90}{$\,\,$ PAIR + LSI}
         & \includegraphics[width=0.15\linewidth]{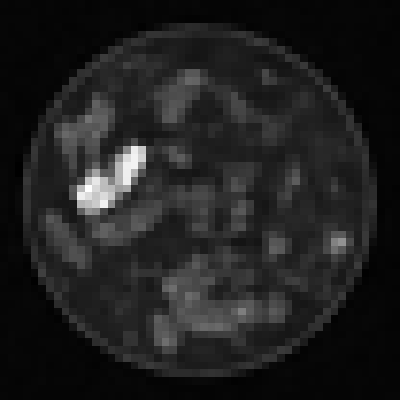}
         & \includegraphics[width=0.15\linewidth]{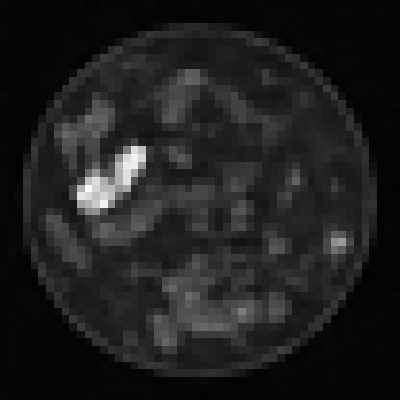}
         & \includegraphics[width=0.15\linewidth]{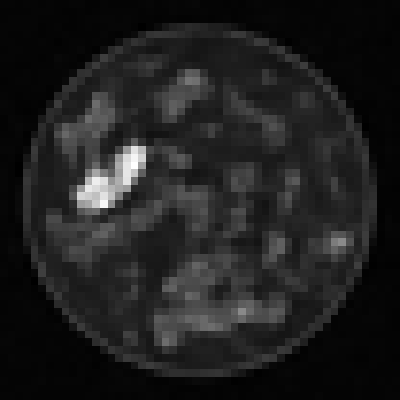}
         & \includegraphics[width=0.15\linewidth]{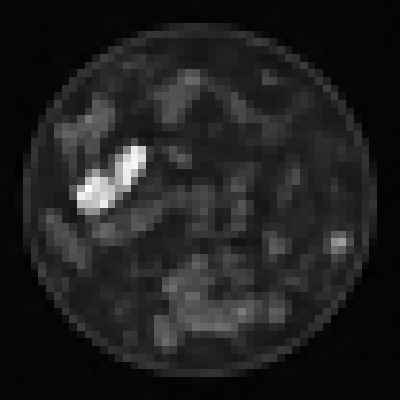}
         & \includegraphics[width=0.15\linewidth]{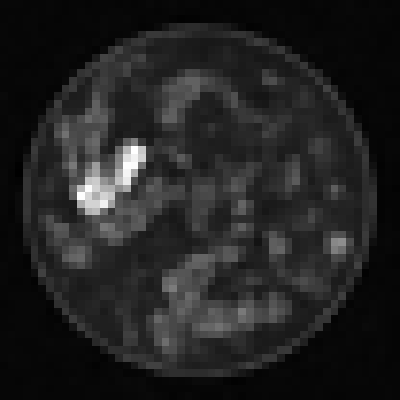}
         \\
         & \textcolor{white}{.} & & & & & \\
         & & & $\bfX$ & & \\
         & & & \includegraphics[width=0.15\linewidth]{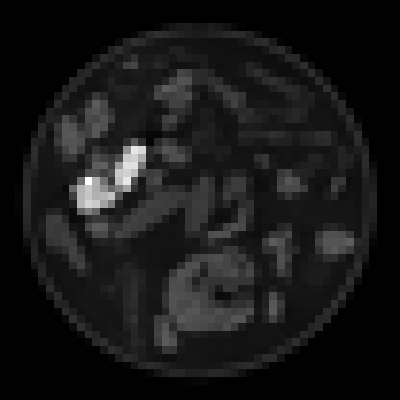} & & &\\
    \end{tabular}
    \caption{Example reconstructions of PAIR, end-to-end encoder decoders, and PAIR + LSI from sinogram $\bfY$, with different projections.  Each column header shows the input sinogram that has been corrupted with random or blocks of missing angles.  The target $\bfX$ is representative of the samples on which all networks were trained.}
    \label{fig:2detect_ex0}
\end{figure}

\begin{figure}
    \centering
    \begin{tabular}{lcccccc}
        & $\bfY$ & $P_{\rm rand}(\bfY)$ & $P_{\rm rect}(\bfY)$ & $\widetilde P_{\rm rand}(\bfY)$ & $\widetilde P_{\rm rect}(\bfY)$ \\
         & \includegraphics[width=0.15\linewidth]{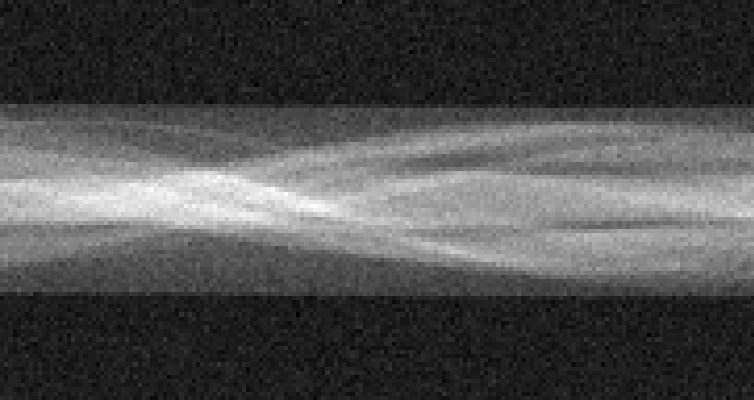}
         & \includegraphics[width=0.15\linewidth]{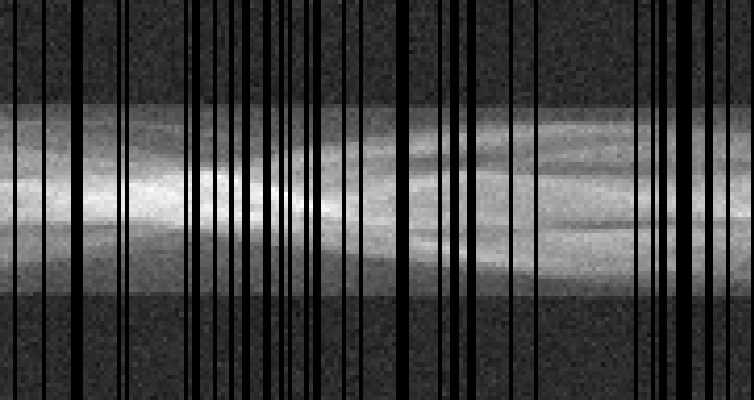}
         & \includegraphics[width=0.15\linewidth]{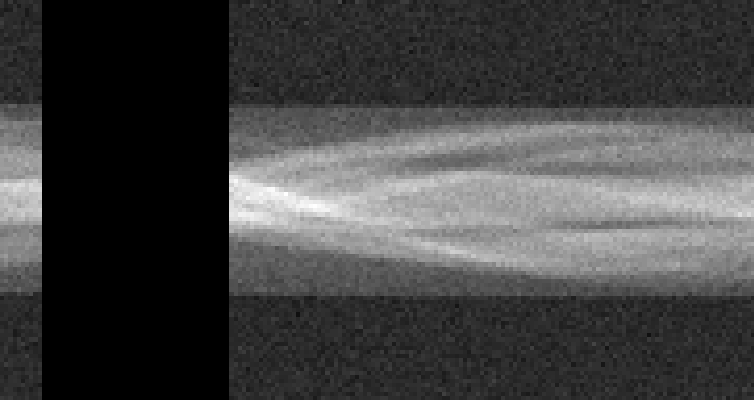}
         & \includegraphics[width=0.15\linewidth]{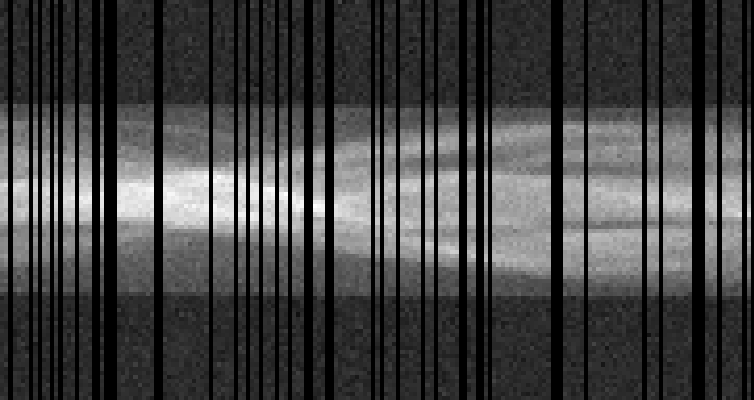}
         & \includegraphics[width=0.15\linewidth]{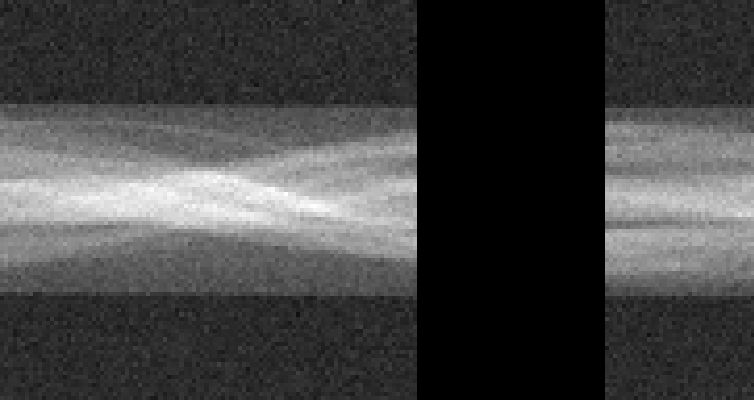}
         \\
         & \textcolor{white}{.} & & & & & \\
         \hline\\
         & \textcolor{white}{.} & & & & & \\
         \rotatebox{90}{$\qquad$ PAIR}
         & \includegraphics[width=0.15\linewidth]{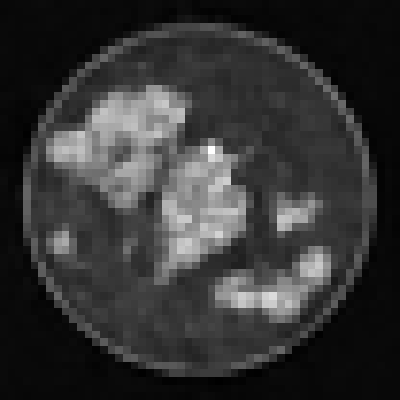}
         & \includegraphics[width=0.15\linewidth]{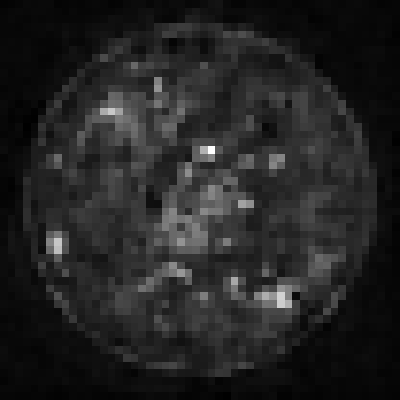}
         & \includegraphics[width=0.15\linewidth]{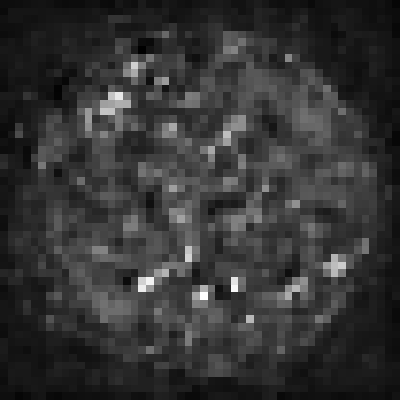}
         & \includegraphics[width=0.15\linewidth]{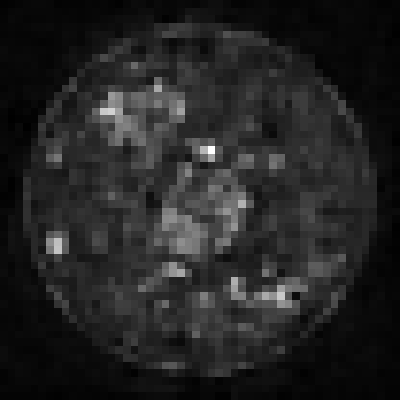}
         & \includegraphics[width=0.15\linewidth]{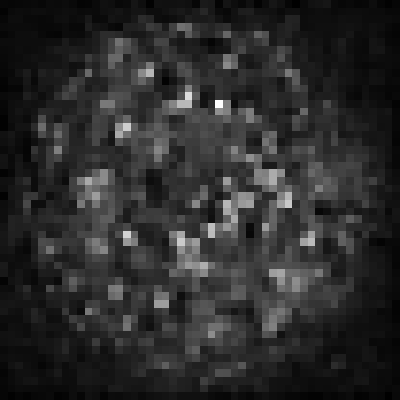}
         \\
         \rotatebox{90}{$\,\,$ En/De-coders}
         & \includegraphics[width=0.15\linewidth]{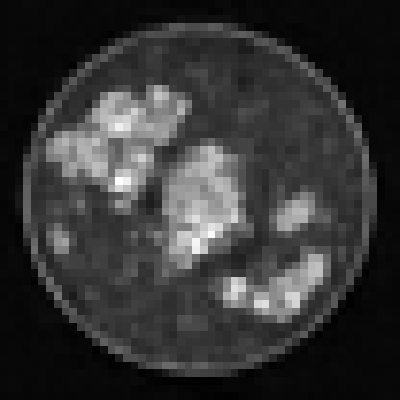}
         & \includegraphics[width=0.15\linewidth]{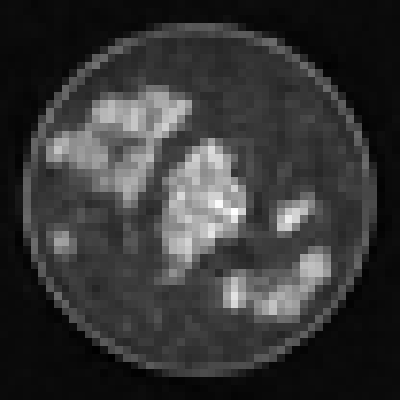}
         & \includegraphics[width=0.15\linewidth]{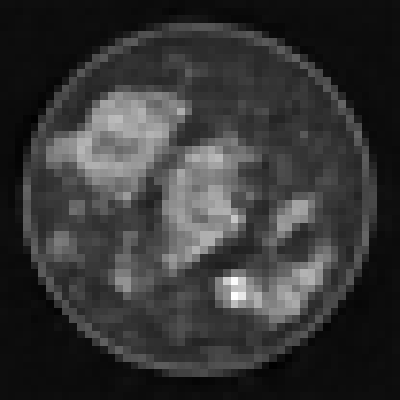}
         & \includegraphics[width=0.15\linewidth]{media/2DeteCT_Examples/ex_0/hat_rand_ED_x_0.png}
         & \includegraphics[width=0.15\linewidth]{media/2DeteCT_Examples/ex_0/hat_rect_ED_x_0.png}
         \\
         \rotatebox{90}{$\,\,$ PAIR + LSI}
         & \includegraphics[width=0.15\linewidth]{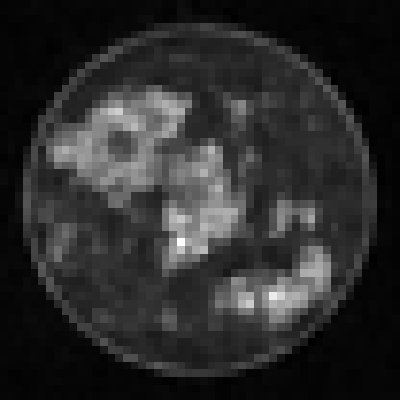}
         & \includegraphics[width=0.15\linewidth]{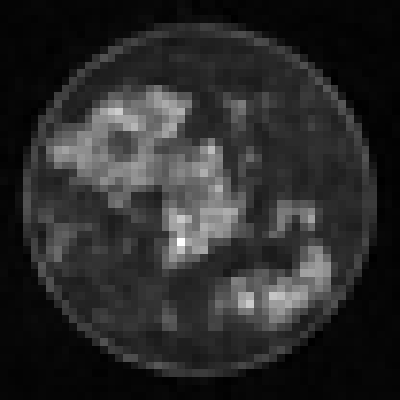}
         & \includegraphics[width=0.15\linewidth]{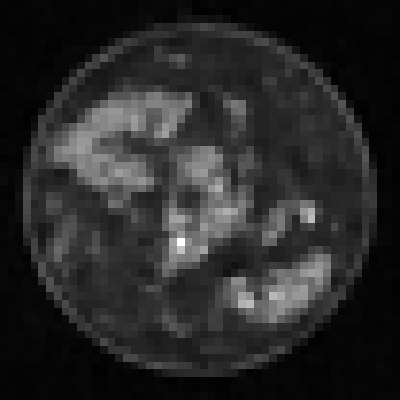}
         & \includegraphics[width=0.15\linewidth]{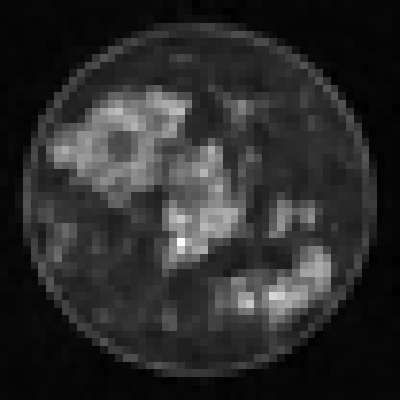}
         & \includegraphics[width=0.15\linewidth]{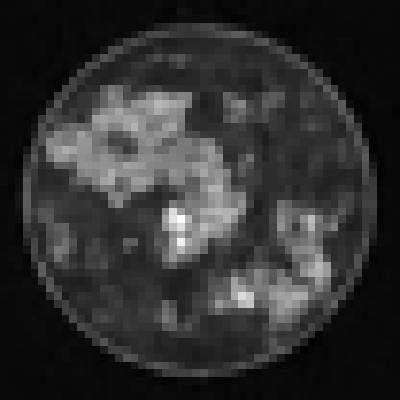}
         \\
         & \textcolor{white}{.} & & & & & \\
         & & & $\bfX$ & & \\
         & & & \includegraphics[width=0.15\linewidth]{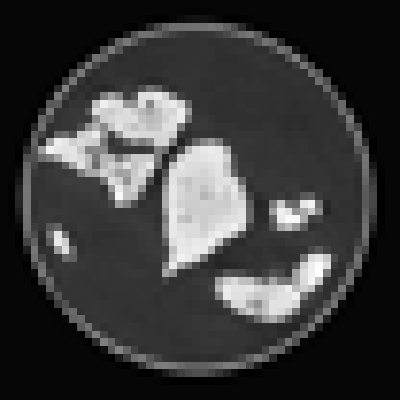} & & &\\
    \end{tabular}
    \caption{Reconstructions of an out-of-distribution test sample using PAIR, end-to-end encoder decoders, and PAIR + LSI from different observations of $\bfY$.  Each column header shows the input sinogram that has been corrupted with random or blocks of missing angles, and the OOD target $\bfX$ is provided for comparison.}
    \label{fig:2detect_ex1_OOD}
\end{figure}
    
In \Cref{tab:results_all}, we used the relative reconstruction error to evaluate the quality of a reconstruction, but this metric is not computable in practice, as we do not have access to the true solution. One of the main advantages of the PAIR framework is that it provides several cheaply computable metrics that can serve as proxies to indicate reconstruction quality \cite{hart2025paired}.  These metrics can also be used to assess whether a new observation lies within the distribution of training data; if the PAIR metrics for the new sample are similar to the metrics observed for the training data, we have some indication we can trust the prediction from the PAIR network.  If not, further investigation and refinement (e.g., PAIR $+$ LSI) may be required. 
In \Cref{fig:OODmetrics}, we use two of these metrics to illustrate the ability to detect corrupted data (in this case, missing angles) and to improve reconstructions using PAIR+LSI.  
Samples are plotted with the relative residual estimate $${\|(d_y \circ m^{\to} \circ e_x)(\bfX_{\rm pred}) - \bfY\|_{\rm F}}/{\|\bfY\|_{\rm F}}$$ on the $y$-axis and the relative autoencoded data difference $${\|d_y(\bfz_y) - \bfY\|_{\rm F}}/{\|\bfY\|_{\rm F}}$$ on the $x$-axis.  Circles represent the metrics for $100$ full-angle test samples (these are clustered in the bottom left corner of the plot), stars correspond to the same $100$ samples with a consecutive block of missing angles (clustered in the top right corner of the plot), and squares represent those same masked samples after applying LSI (clustered between the other two).  For in-distribution test samples with full angles (circles), $\bfX_{\rm pred} = (d_x\circ m^{\gets} \circ e_y)(\bfY)$ and $\bfz_y=e_y(\bfY)$; for masked angle samples (stars), $\bfX_{\rm pred} = (d_x\circ m^{\gets} \circ e_y\circ P_{\rm rect}(\bfY))$ and $\bfz_y=(e_y\circ P_{\rm rect})(\bfY)$; and for masked angles samples with LSI (squares), $\bfX_{\rm pred} = (d_x\circ m^{\gets})(\widehat\bfz_y)$ and $\bfz_y=\widehat\bfz_y$, where $\hat \bfz_y \in \argmin_{\bfz_y} \left\| (P_{\rm rect}\circ d_y)(\bfz_y)- \bfY_\text{sub} \right\|_{\rm F}$ is found using an L-BFGS optimizer with Strong-Wolfe line search for ten iterations.  Grey lines connect corresponding missing angle samples before and after LSI, and the color of each point indicates the RRE.

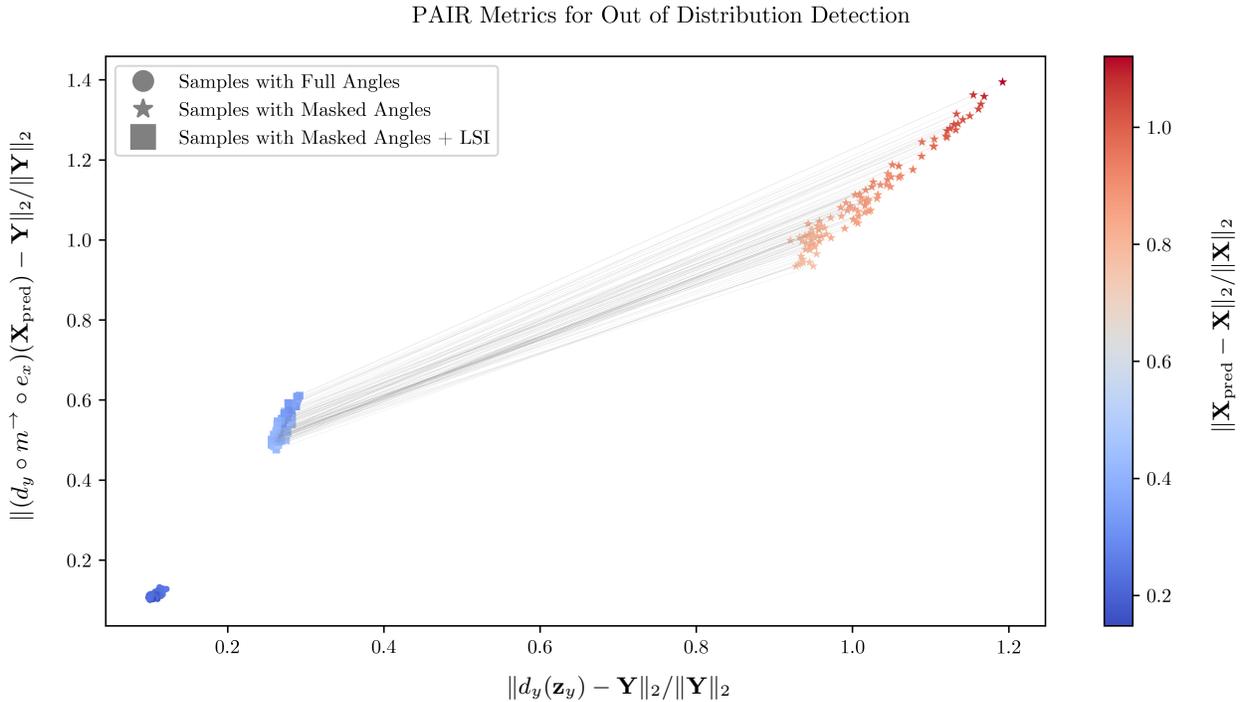
\begin{figure}[ht]
    \centering
    \begin{tikzpicture}
        \node[anchor=south west, inner sep=0] (img)
            {\includegraphics[width=0.9\linewidth, trim={0.7cm 0.8cm 2.2cm 0cm}, clip]{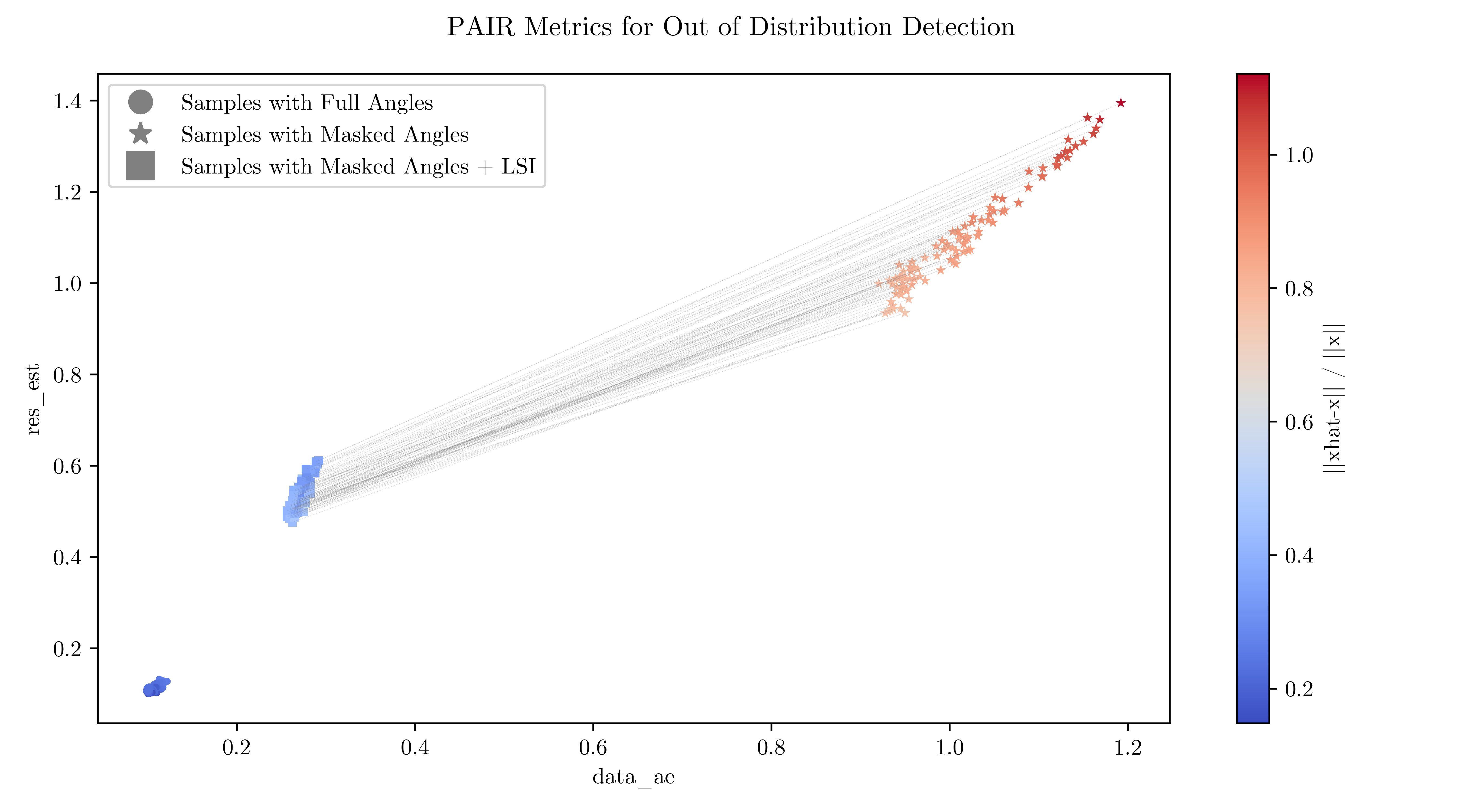}};

        \node[below=0.1cm of img.south] {\small ${\|d_y(\bfz_y) - \bfY\|_2}/{\|\bfY\|_2}$};

        \node[rotate=90, anchor=center]
            at ([xshift=-0.5cm]img.west|-img.center)
            {\small ${\|(d_y \circ m^{\to} \circ e_x)(\bfX_{\rm pred}) - \bfY\|_2}/{\|\bfY\|_2}$};

        \node[rotate=90, anchor=center]
            at ([xshift=0.6cm]img.east|-img.center)
            {\small $\|\bfX_{\rm pred}- \bfX\|_2 / \|\bfX\|_2$};
    \end{tikzpicture}

    \caption{PAIR out-of-distribution metrics for in-distribution samples with full observational data, samples with a block of masked angles, and samples with a block of masked angles reconstructed using PAIR+LSI.}
    \label{fig:OODmetrics}
\end{figure}

We observe three well-separated spatial clusters: one for the full angle samples, one for the masked angle samples with PAIR alone, and one for the masked angle samples with PAIR $+$ LSI. The errors are also clearly clustered;  the error is highest on samples with masked angles when PAIR is used without LSI, and the error is lowest when PAIR is used with full angles.  Although it does not achieve the same accuracy as in the full angle case, PAIR $+$ LSI clearly reduces the reconstruction errors associated with missing angles, since the cluster of squares is notably closer to the cluster of circles.

\subsection{Example 2: Cross-Well Seismic Inversion} \label{sec:seismic}

Missing or incomplete data are major problems that also arise in applications such as seismic inversion, known as Full Waveform Inversion (FWI, \cite{tarantola2005inverse,VirieuxOperto2009}). Due to the cost of deploying the seismic source and receivers, data acquisition may be deliberately incomplete \cite{SubsamplingHerrmann, SubsamplingBaumstein, haber2012effective}. 
This example serves two purposes: to show that the proposed PAIR framework with data reconstruction also applies to nonlinear inverse problems, and to extend the proposed approach to probabilistic settings, as introduced in \cite{chung2025good}, yielding ensembles of data and model reconstructions through the use of Variational Paired Autoencoders (VPAE), \cite{kingma2019introduction}. In this probabilistic setting, the latent variables $\mathbf{z}$ are now represented by random variables $\mathbf{z} \sim \mathcal{N}(\boldsymbol{\mu},\diag{\boldsymbol{\sigma}})$ characterized by mean $\boldsymbol{\mu}$ and $\boldsymbol{\sigma}$.

This experiment considers cross-well seismic imaging \cite{PrattCrossHoleTomo, CrossHoleSeisTomo} in transmission mode, meaning that seismic sources and receivers are in opposite wells or boreholes.
Consider the acoustic wave equation,
\begin{equation}
    \nabla^2 u(x,t)  - \frac{1}{v(x)^2} \partial_{tt} u(x,t) =  s(x,t),
\end{equation}
where a spatial and temporally dependent source term $s(x,t)$ generates a seismic wave field $u(x,t)$ in the p-wave velocity model $v(x)$. 
For each of the $30$ source locations ($15$ per borehole), we numerically solve the discretized wave equation using DeepWave \cite{richardson_alan_2023}, and the observed data are measured by $48$ receivers located in the borehole on the opposite side of the velocity model.  For readability, we use the same notation for continuous fields and their discretized representations, as all quantities are ultimately treated in discrete form for numerical simulation and learning. For a single velocity model, the observed data $\bfY = Q(u)$ contains observations for all time points from all $30$ sources/channels.
See \Cref{fig:seismic_data_models}(a) for five samples of the $640 \si{m}  \times 800 \si{m}$ velocity models, and see \Cref{fig:seismic_data_models}(b) for three of the 30 channels obtained from a single velocity model.

\begin{figure}[htbp]
    \centering
    \begin{subfigure}[b]{0.75\textwidth}
        \centering
        \includegraphics[width=\textwidth]{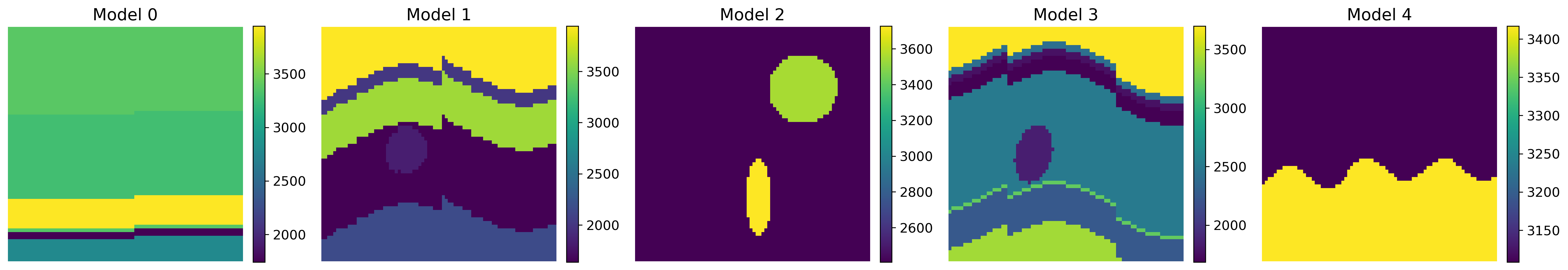}
        \caption{Five out of \num{2000} velocity models from the dataset.}
        \label{fig:subfig1}
    \end{subfigure}
    \\
    \begin{subfigure}[b]{0.75\textwidth}
        \centering
        \includegraphics[width=\textwidth]{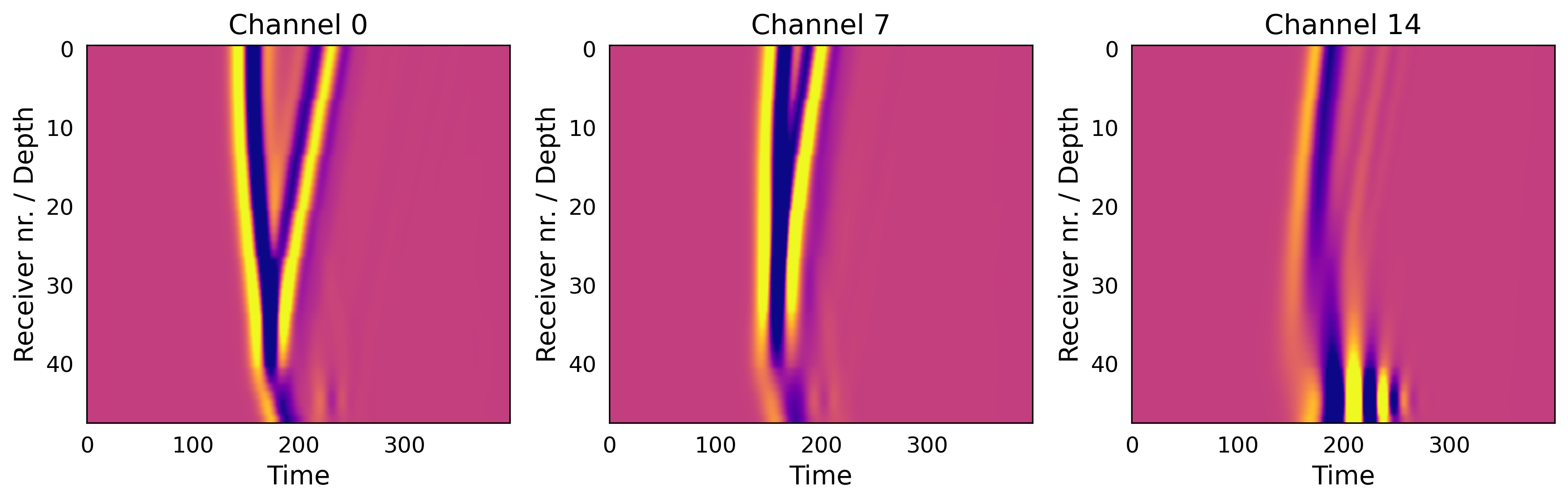}
        \caption{Full data for a subset of the sources, for a single velocity model.}
        \label{fig:subfig2}
    \end{subfigure}
    \caption{Seismic FWI example: (a) sample velocity models, (b) full (uncorrupted) data for three sources/channels.}
\label{fig:seismic_data_models}
\end{figure}

We generate training data containing corresponding data/velocity samples $\{\bfY^{\rm train}_i, \bfX^{\rm train}_i\}_{i=1}^{\num{2000}}$, where $\bfX^{\rm train}_i$ denotes a discretized velocity field.  We train the VPAE to minimize the combined loss given in \Cref{eq:PAIRloss}. 
Both the data and velocity models use mean-shift and standard-deviation scaling, where the normalization constants are computed globally over the full training set. The paired autoencoders are trained on complete data, while all inference uses data with missing receiver signals, denoted as $\mathbf{Y}_{\text{sub}} = P(\mathbf{Y})$. For one channel corresponding to one of the velocity models from the validation set, we provide in the top row of \Cref{fig:seismic_data_reconstruction} the full data $\mathbf{Y}$ and corrupted data $\mathbf{Y}_{\text{sub}}$ corresponding to $90\%$ random missing receivers. 

The VPAE design for this example employs a fully convolutional multi-level ResNet for both the data and the model encoder-decoder structures. 
Given the network state $\bfS_k$ at layer index $k$, the state is processed by an atrous spatial pyramid pooling (ASPP) block \cite{chen2017deeplab}, which contains convolutions, normalization, and non-linearity, and also changes the number of channels at the end of each of the three levels. After each level, the number of channels is doubled, while the spatial resolution is halved. To maintain the ResNet flow across levels, we also change the number of channels of the state via a $1\times1$ convolution, $\bfC_k \bfS_k$. We denote the identity or pooling/subsampling operator as $P_k$. Together the encoder design reads

\begin{equation}\label{eq:resnetblock}
    \bfS_{k+1} = P_k (\bfC_k \bfS_k - \operatorname{ASPP}_k(\bfS_k)).
\end{equation}

We use three levels, with three layers each, so $P_k$ is a pooling operator at $k=3,6,9$ and $P_k$ is equal to the identity map for all other indices. Similarly, $\bfC_k$ changes the number of channels of the state for layers $k=3,6,9$, while acting as the identity operator for all other indices. The number of channels is $64 \rightarrow 128 \rightarrow 256$ for the three levels.
The decoder follows the reverse design of \Cref{eq:resnetblock}. In the latent space, the variables originating from data and velocity models typically have different sizes and channel counts. The latent-space operators $m^{\to}$ and $m^{\gets}$ are linear transformations and translations that act on vectorized latent variables while preserving the batch dimension. 

We compare three approaches for inference from the corrupted data that all use (pieces of) the same pre-trained VPAE:

\begin{enumerate}

    \item \emph{PAIR:} we use the variational paired autoencoders alone, without any explicit treatment of the corrupted or missing observations. The velocity model is estimated directly via the learned surrogate inversion map by encoding the incomplete data and decoding through the paired latent mappings, as in \Cref{eq:PAIRinverse}.
    
    \item \emph{Model-based LSI (M-LSI):} we directly estimate a velocity model via solving the inverse problem, parameterized by the model decoder. See, e.g., \cite{bora2017compressed, pmlr-v119-asim20a} and \cite{NEURIPS20181bc2029a, mosser2020stochastic, zhu2022integrating, he2021reparameterized, IwakiriAutoSeismic} for deterministic and probabilistic applications to linear and nonlinear problems, respectively. 
    This approach is formulated in terms of the partial observation operator $P$ and the forward operator $A = Q \circ F$, where $F$ represents the velocity-to-wavefield map and $Q$ denotes the full observation operator.

    Using a maximum of $100$ L-BFGS iterations we compute
    \begin{equation}\label{eq:mslsi}
        \hat{\bfz}_x = \argmin_{\bfz_x} \| (P \circ A \circ d_x) (\bfz_x) - \bfy_\text{sub} \|_2^2, \quad \text{with} \quad \hat{\bfx} = d_x(\hat{\bfz}_x).
    \end{equation}
    Note that \Cref{eq:mslsi} defines a deterministic optimization problem, as it does not involve the probabilistic formulation of the encoder and propagates only through the decoder. As an initial guess in the latent space, we employ the model encoder $e_x$ to encode the pixel-wise mean velocity model from the training set. The numerical experiments solve the above problem for an ensemble of random initial guesses, generated by sampling the encoded mean velocity model, to obtain an ensemble of the QoI estimates. Note that the samples in this ensemble are not true posterior samples, since they are obtained by solving the deterministic optimization problem \Cref{eq:mslsi} rather than by sampling a Bayesian posterior; nevertheless, variability across the ensemble provides a qualitative indication of uncertainty in the estimates.
    
    \item \emph{PAIR + LSI:} we perform latent space inversion in the data space (\Cref{eq:latentcomp_zy_P}) using a maximum of
    $25$ iterations of L-BFGS to reconstruct missing data, $\hat \bfz_y \in \argmin_{\bfz_y} \left\| (P\circ d_y)(\bfz_y)- \bfy_\text{sub} \right\|$ followed by decoding (the estimated latent data variables) to the velocity model, i.e., $\hat \bfx = (d_x\circ m^{\gets}) (\hat \bfz_y)$. The initial guess is the encoding of the missing data, $e_y(\bfy_\text{sub})$. Because this approach uses the VPAE, the encoded missing data is sampled in the latent space, which yields an ensemble of data reconstructions and an ensemble of estimated velocity models.
\end{enumerate}

Regarding computational cost, Method $2$ stands out as computationally expensive because every iteration requires evaluating $A$ and its adjoint, which requires solving PDEs. Methods $1$ and $3$ rely solely on the surrogate models provided by the trained variational paired autoencoders.

First, in \Cref{fig:seismic_data_reconstruction}, we illustrate reconstructions of missing data observations from the $90\%$ incomplete dataset using PAIR+LSI.
The VPAE constructs an ensemble of data reconstructions $\hat \bfz_y$. For visualization purposes, we show five out of the \num{25} decoded samples of the data reconstructions $d_y(\hat \bfz_y)$, along with the average and standard deviation of the observation reconstructions. Note that while PAIR+LSI does not estimate $d_y(\hat \bfz_y)$, it could be seen as an intermediate quantity.
We observe that PAIR+LSI provides good reconstructions of the data from the missing receiver signals.  Along the same lines, we expect that velocity field reconstructions, $\hat \bfx = (d_x\circ m^{\gets}) (\hat \bfz_y)$, will yield accurate approximations of the true velocity fields.  

\begin{figure}
    \centering
    \includegraphics[width=1\linewidth]{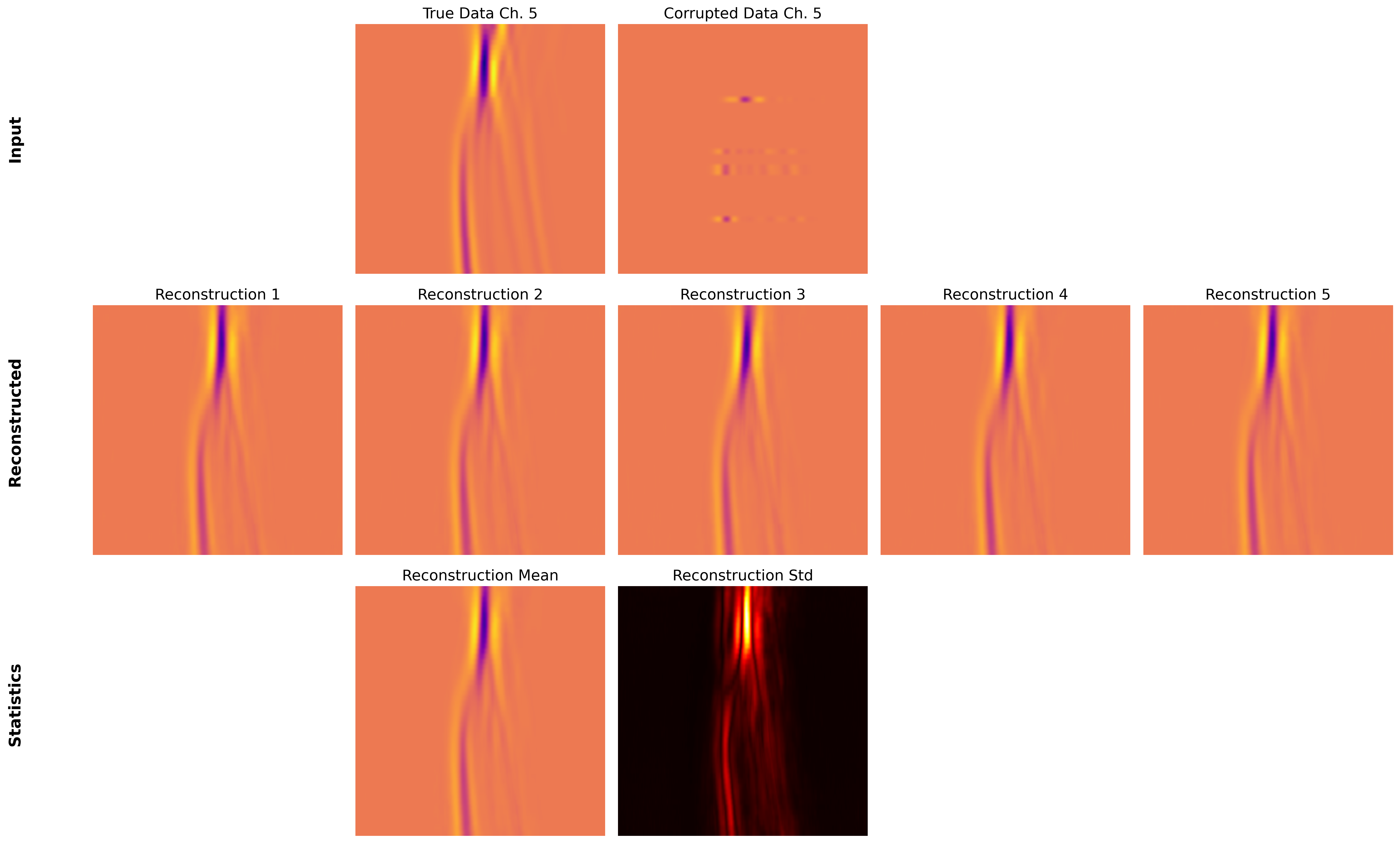}
    \caption{Top row: Observed data for one source, with full data (left) and with $90\%$ random missing receivers (right). Middle row: Five samples of data reconstructions using PAIR+LSI. Bottom row: Average and standard deviation of data reconstructions. Data reconstructions in the physical space appear as an intermediate quantity only.}
    \label{fig:seismic_data_reconstruction}
\end{figure}

Next, we investigate and compare velocity reconstructions obtained via inference from the corrupted data. In \Cref{fig:seismic_model_reconstruction}, we provide sample model reconstructions along with standard deviations (in the last row) for reconstructions from each of the estimation approaches, with the true model provided for comparison. 
We observe that for $90\%$ missing data, the velocity reconstructions are approximately blurred versions of the true model (see the last column of \Cref{fig:seismic_model_reconstruction}).  Meanwhile, both Methods 1 and 2 that directly use the corrupted data lead to models that are, to a large extent, fundamentally different from the true velocity model.  Although M-LSI with incomplete data does take into account the data acquisition mask, solving the non-convex optimization problem from samples of the encoding of the mean-model, in combination with extremely sparse data acquisition, leads to a clear data acquisition imprint and oftentimes little resemblance to the true model. We note that M-LSI has been successfully used in various applications, including seismic inversion, but those results typically used a reasonably accurate initial guess. Here, the dataset was designed so that the per-pixel mean of the velocity models was uniform everywhere, resulting in an initial guess that is far from the true model in most cases.

\begin{figure}
    \centering
    \includegraphics[width=0.55\linewidth]{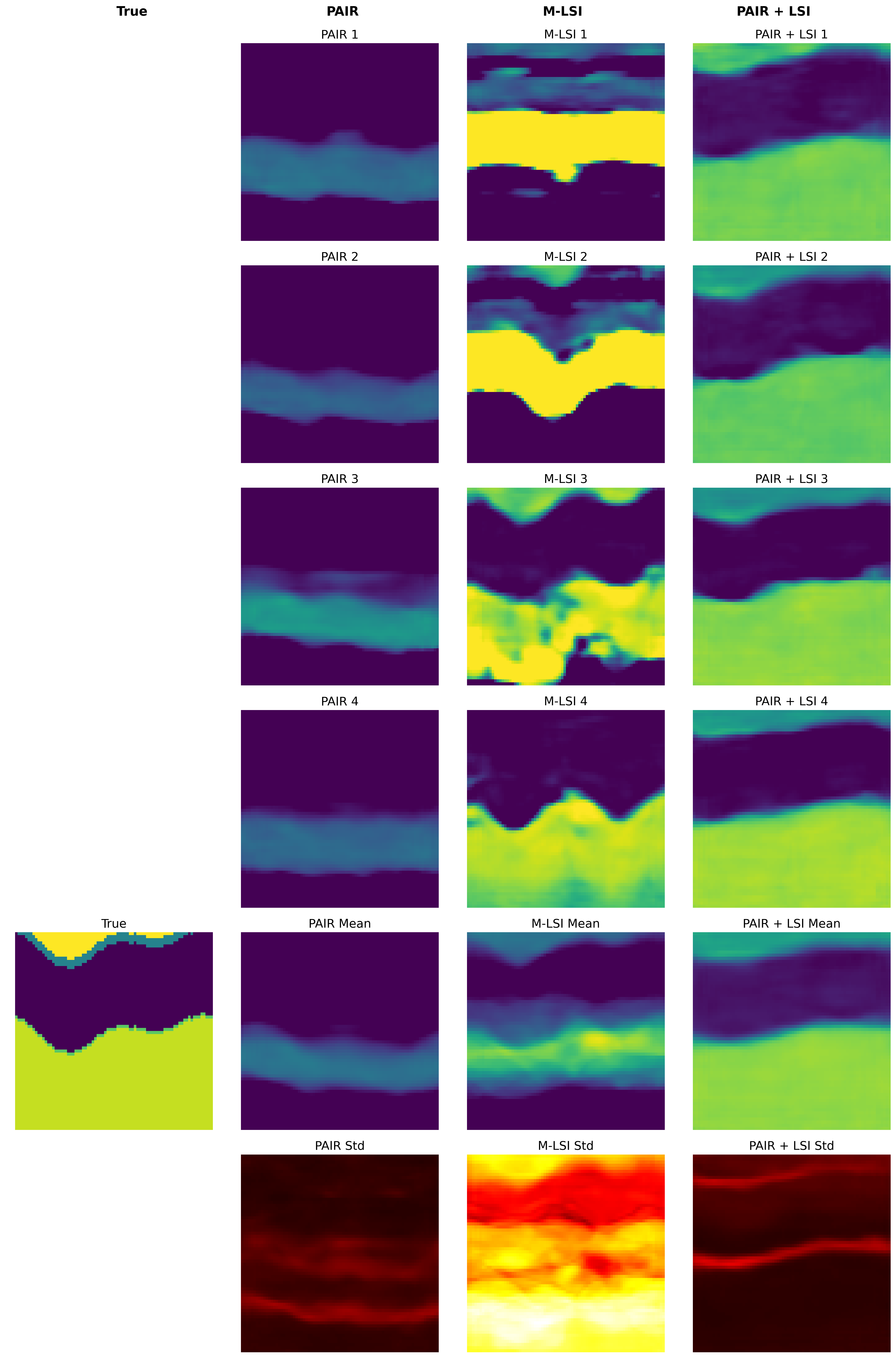}
    \caption{Seismic FWI velocity reconstructions from corrupted data with $90\%$ missing receivers. Left column: One example of a velocity model from the validation set.  Second column (PAIR): A few samples generated by the VPAE from directly inputting corrupted data, as shown in Figure \ref{fig:seismic_data_reconstruction}. Third column (M-LSI): Result of estimating a model using a model decoder, starting from samples of the encoding of the mean velocity model from the training set. Fourth column: (PAIR+LSI): A few samples generated by the VPAE from data that was reconstructed using the same VPAE via LSI. The last row corresponds to the standard deviations for the reconstructions.}
    \label{fig:seismic_model_reconstruction}
\end{figure}

\begin{figure}
    \centering
    \includegraphics[width=1\linewidth]{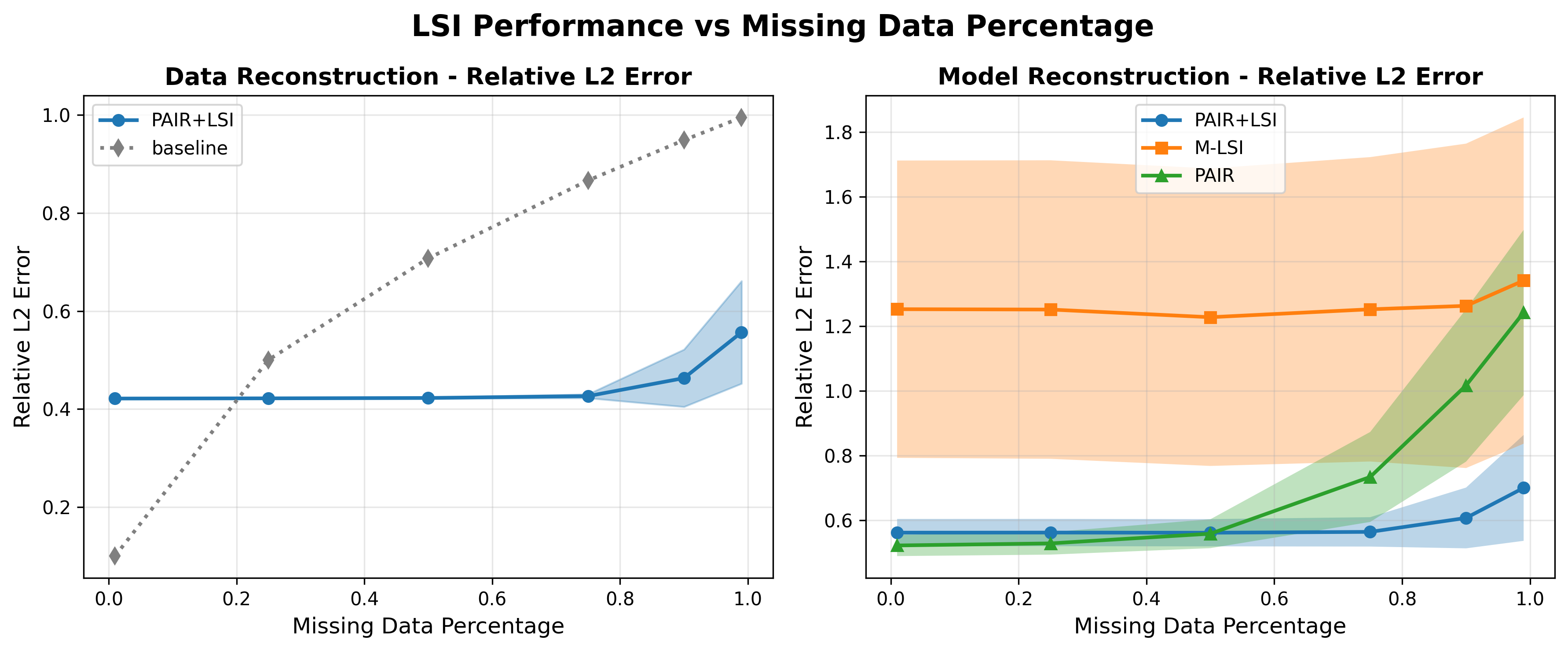}
    \caption{For different percentages of missing data in the seismic example, we provide error statistics for the data reconstructions (left plot) and the velocity model reconstructions (right plot). The data baseline error is the difference between observed and full data.  The statistics are averaged over $100$ velocity models, where we generate $25$ samples per velocity model. Plotted errors are based on the sample mean of the errors (smaller numbers are better). Shading indicates the mean $\pm 1$ standard deviation.}
    \label{fig:LSI_metrics_seismic}
\end{figure}

Next, we investigate the performance of the inference approaches for different percentages of missing data. That is, we vary the amount of missing observations and provide in \Cref{fig:LSI_metrics_seismic} average reconstruction errors for data reconstructions and model reconstructions for PAIR+LSI. From the left plot, we observe that the error in the reconstructed data is near constant for moderately corrupted data, while the mean and standard deviation of the reconstructed data error increases more rapidly for highly corrupted data. The reconstructed data error is not zero for small amounts of missing data, because the data decoder $d_x$ is not able to reconstruct the observed data exactly, either due to a lack of training data, or an insufficiently expressive decoder. 
In the right plot, we provide model reconstruction errors for PAIR+LSI for different percentages of missing data, with the corresponding results for PAIR and M-LSI provided for comparison.  Model reconstruction statistics show that PAIR performs similarly to PAIR+LSI for a moderate amount of missing data, whereas when few data are available (i.e., larger percentages of missing data), PAIR+LSI outperforms the other methods. M-LSI performs poorly in general, which is primarily due to the absence of a good initial guess in the latent space, contrary to the PAIR and PAIR+LSI approaches.

\section{Conclusions}
\label{sec:conclusion}
In this work, we describe a new approach for latent space inference that can handle observational inconsistencies or missing data.  By exploiting a paired autoencoder framework and performing inversion in the latent space, we show that more accurate reconstructions can be obtained for problems with missing observational data.
We extend the capabilities of the PAIR framework to handle missing data, and we show that the PAIR+LSI results are superior to direct mapping via PAIR, as well as encoder-decoder networks of incomplete data, especially if the incomplete data are different from those used for training.

Future work includes exploring the use of latent space inference with PAIR for optimal experimental design (e.g., identifying the important angles to take measurements or the important sensors for sensor placement).  In contrast to works like \cite{siddiqui2024deep} that combine machine learning techniques with OED to avoid bilevel optimization problems, latent space inversion for the inner problem would provide a fast reconstruction approximation for different missing angles. Latent space inference could also be used for timely solves in hyperparameter estimation or in scenarios with the forward model or noise model is different from the assumptions.

\printbibliography

@misc{richardson_alan_2023,
  author       = {A. Richardson},
  title        = {Deepwave},
  year         = 2025,
  publisher    = {Zenodo},
  version      = {v0.0.25},
  doi          = {10.5281/zenodo.17924647}
}

@article{IwakiriAutoSeismic,
   author = "Iwakiri, H. and Mizuno, N. and Shibayama, T. and Kinoshita, A",
   title = "Full-Waveform Inversion in Optimization-Friendly Latent Space Created by a Deep Neural Network", 
   journal= "",
   year = "2022",
   volume = "2022",
   number = "1",
   pages = "1-5",
   doi = "10.3997/2214-4609.202210381",
   url = "https://www.earthdoc.org/content/papers/10.3997/2214-4609.202210381",
   publisher = "European Association of Geoscientists &amp; Engineers",
   issn = "2214-4609",
   type = "",
  }

@article{kingma2019introduction,
  title={An introduction to variational autoencoders},
  author={Kingma, Diederik P and Welling, Max and others},
  journal={Foundations and Trends{\textregistered} in Machine Learning},
  volume={12},
  number={4},
  pages={307--392},
  year={2019},
  publisher={Now Publishers, Inc.}
}

@InProceedings{pmlr-v119-asim20a,
  title = 	 {Invertible generative models for inverse problems: mitigating representation error and dataset bias},
  author =       {Asim, Muhammad and Daniels, Mara and Leong, Oscar and Ahmed, Ali and Hand, Paul},
  booktitle = 	 {Proceedings of the 37th International Conference on Machine Learning},
  pages = 	 {399--409},
  year = 	 {2020},
  editor = 	 {III, Hal Daumé and Singh, Aarti},
  volume = 	 {119},
  series = 	 {Proceedings of Machine Learning Research},
  month = 	 {7},
  publisher =    {PMLR},
  url = 	 {https://proceedings.mlr.press/v119/asim20a.html}
}

@inproceedings{NEURIPS20181bc2029a,
 author = {Hand, Paul and Leong, Oscar and Voroninski, Vlad},
 booktitle = {Advances in Neural Information Processing Systems},
 editor = {S. Bengio and H. Wallach and H. Larochelle and K. Grauman and N. Cesa-Bianchi and R. Garnett},
 pages = {},
 publisher = {Curran Associates, Inc.},
 title = {Phase Retrieval Under a Generative Prior},
 url = {https://proceedings.neurips.cc/paper_files/paper/2018/file/1bc2029a8851ad344a8d503930dfd7f7-Paper.pdf},
 volume = {31},
 year = {2018}
}

@article{zhu2022integrating,
  title={Integrating deep neural networks with full-waveform inversion: Reparameterization, regularization, and uncertainty quantification},
  author={Zhu, Weiqiang and Xu, Kailai and Darve, Eric and Biondi, Biondo and Beroza, Gregory C},
  journal={Geophysics},
  volume={87},
  number={1},
  pages={R93--R109},
  year={2022},
  publisher={Society of Exploration Geophysicists}
}

@article{he2021reparameterized,
  title={Reparameterized full-waveform inversion using deep neural networks},
  author={He, Qinglong and Wang, Yanfei},
  journal={Geophysics},
  volume={86},
  number={1},
  pages={V1--V13},
  year={2021},
  publisher={Society of Exploration Geophysicists}
}

@inproceedings{bora2017compressed,
  title={Compressed sensing using generative models},
  author={Bora, Ashish and Jalal, Ajil and Price, Eric and Dimakis, Alexandros G},
  booktitle={International conference on machine learning},
  pages={537--546},
  year={2017},
  organization={PMLR}
}

@article{SubsamplingBaumstein,
    author = {Baumstein, Anatoly and Neelamani, Ramesh},
    title = {Accurate data reconstruction through simultaneous application of statistical and physics-based constraints to multiple geophysical data sets},
    journal = {Geophysics},
    volume = {75},
    number = {6},
    pages = {WB165-WB172},
    year = {2010},
    month = {12},
    issn = {0016-8033},
    doi = {10.1190/1.3481653},
    url = {https://doi.org/10.1190/1.3481653},
    eprint = {https://pubs.geoscienceworld.org/seg/geophysics/article-pdf/75/6/WB165/3235444/gsgpy_75_6_WB165.pdf},
}

@article{SubsamplingHerrmann,
    author = {Herrmann, Felix J.},
    title = {Randomized sampling and sparsity: Getting more information from fewer samples},
    journal = {Geophysics},
    volume = {75},
    number = {6},
    pages = {WB173-WB187},
    year = {2010},
    month = {12},
    issn = {0016-8033},
    doi = {10.1190/1.3506147},
    url = {https://doi.org/10.1190/1.3506147},
    eprint = {https://pubs.geoscienceworld.org/seg/geophysics/article-pdf/75/6/WB173/3235525/gsgpy_75_6_WB173.pdf},
}

@article{CrossHoleSeisTomo,
    author = {Wang, Yanghua and Rao, Ying},
    title = {Crosshole seismic waveform tomography – I. Strategy for real data application},
    journal = {Geophysical Journal International},
    volume = {166},
    number = {3},
    pages = {1224-1236},
    year = {2006},
    month = {09},
    issn = {0956-540X},
    doi = {10.1111/j.1365-246X.2006.03030.x},
    url = {https://doi.org/10.1111/j.1365-246X.2006.03030.x},
    eprint = {https://academic.oup.com/gji/article-pdf/166/3/1224/6099940/166-3-1224.pdf},
}

@article{mosser2020stochastic,
  title={Stochastic seismic waveform inversion using generative adversarial networks as a geological prior},
  author={Mosser, Lukas and Dubrule, Olivier and Blunt, Martin J},
  journal={Mathematical Geosciences},
  volume={52},
  number={1},
  pages={53--79},
  year={2020},
  publisher={Springer}
}

@article{PrattCrossHoleTomo,
author = {Pratt, R. Gerhard and Worthington, M. H.},
title = {Inverse Theory Applied to Multi-source Cross-hole Tomography.},
journal = {Geophysical Prospecting},
volume = {38},
number = {3},
pages = {287-310},
doi = {10.1111/j.1365-2478.1990.tb01846.x},
url = {https://onlinelibrary.wiley.com/doi/abs/10.1111/j.1365-2478.1990.tb01846.x},
year = {1990}
}

@book{tarantola2005inverse,
  title={Inverse Problem Theory and Methods for Model Parameter Estimation},
  author={Tarantola, Albert},
  year={2005},
  publisher={SIAM}
}

@article{VirieuxOperto2009,
	Author = {J. Virieux and S. Operto},
	Journal = {Geophysics},
	Number = {6},
	Title = {An overview of full-waveform inversion in exploration geophysics},
	Volume = {74},
	Year = {2009}}

@article{chen2017deeplab,
  title={Deep{L}ab: Semantic image segmentation with deep convolutional nets, atrous convolution, and fully connected {CRF}s},
  author={Chen, Liang-Chieh and Papandreou, George and Kokkinos, Iasonas and Murphy, Kevin and Yuille, Alan L},
  journal={IEEE Transactions on Pattern Analysis and Machine Intelligence},
  volume={40},
  number={4},
  pages={834--848},
  year={2017},
  publisher={IEEE}
}

@article{siddiqui2024deep,
  title={Deep optimal experimental design for parameter estimation problems},
  author={Siddiqui, Md Shahriar Rahim and Rahmim, Arman and Haber, Eldad},
  journal={Physica Scripta},
  volume={100},
  number={1},
  pages={016005},
  year={2024},
  publisher={IOP Publishing}
}

@book{natterer2001mathematics,
  title={The mathematics of computerized tomography},
  author={Natterer, Frank},
  year={2001},
  publisher={SIAM}
}

@article{gazzola2019ir,
  title={{IR Tools}: a {MATLAB} package of iterative regularization methods and large-scale test problems},
  author={Gazzola, Silvia and Hansen, Per Christian and Nagy, James G},
  journal={Numerical Algorithms},
  volume={81},
  number={3},
  pages={773--811},
  year={2019},
  publisher={Springer}
}

@article{chung2025good,
  title={Good Things Come in Pairs: Paired Autoencoders for Inverse Problems},
  author={Chung, Matthias and Peters, Bas and Solomon, Michael},
  journal={arXiv preprint arXiv:2505.06549},
  year={2025}
}

@article{chung2025latent,
  title={Latent Twins},
  author={Chung, Matthias and Verma, Deepanshu and Collins, Max and Subrahmanya, Amit N and Sastry, Varuni Katti and Rao, Vishwas},
  journal={arXiv preprint arXiv:2509.20615},
  year={2025}
}

@article{hart2025paired,
  title={A Paired Autoencoder Framework for Inverse Problems via Bayes Risk Minimization},
  author={Hart, Emma and Chung, Julianne and Chung, Matthias},
  journal={arXiv preprint arXiv:2501.14636},
  year={2025}
}

@article{haber2012effective,
  title={An effective method for parameter estimation with PDE constraints with multiple right-hand sides},
  author={Haber, Eldad and Chung, Matthias and Herrmann, Felix},
  journal={SIAM Journal on Optimization},
  volume={22},
  number={3},
  pages={739--757},
  year={2012},
  publisher={SIAM}
}

@book{hansen2010discrete,
  title={Discrete inverse problems: insight and algorithms},
  author={Hansen, Per Christian},
  year={2010},
  publisher={SIAM}
}

@book{kaipio2006statistical,
  title={Statistical and computational inverse problems},
  author={Kaipio, Jari and Somersalo, Erkki},
  volume={160},
  year={2006},
  publisher={Springer Science \& Business Media}
}

@book{engl1996regularization,
  title={Regularization of inverse problems},
  author={Engl, Heinz Werner and Hanke, Martin and Neubauer, Andreas},
  volume={375},
  year={1996},
  publisher={Springer Science \& Business Media}
}

@book{bertero2021introduction,
  title={Introduction to Inverse Problems in Imaging},
  author={Bertero, Mario and Boccacci, Patrizia and De Mol, Christine},
  year={2021},
  publisher={CRC press}
}

@book{epstein2007introduction,
  title={Introduction to the mathematics of medical imaging},
  author={Epstein, Charles L},
  year={2007},
  publisher={SIAM}
}

@book{zhdanov2002geophysical,
  title={Geophysical inverse theory and regularization problems},
  author={Zhdanov, Michael S},
  volume={36},
  year={2002},
  publisher={Elsevier}
}

@book{haber2014computational,
  title={Computational methods in geophysical electromagnetics},
  author={Haber, Eldad},
  year={2014},
  publisher={SIAM}
}

@book{doicu2010numerical,
  title={Numerical regularization for atmospheric inverse problems},
  author={Doicu, Adrian and Trautmann, Thomas and Schreier, Franz},
  year={2010},
  publisher={Springer Science \& Business Media}
}

@misc{sgattoni2024physics,
      title={A physics-aware data-driven surrogate approach for fast atmospheric radiative transfer inversion}, 
      author={Cristina Sgattoni and Luca Sgheri and Matthias Chung},
      year={2024},
      eprint={2410.22609},
      archivePrefix={arXiv},
      primaryClass={physics.ao-ph},
      url={https://arxiv.org/abs/2410.22609}, 
}

@article{guillemot2013image,
  title={Image inpainting: Overview and recent advances},
  author={Guillemot, Christine and Le Meur, Olivier},
  journal={IEEE Signal Processing Magazine},
  volume={31},
  number={1},
  pages={127--144},
  year={2013},
  publisher={IEEE}
}

@book{tanaka1998inverse,
  title={Inverse problems in engineering mechanics},
  author={Tanaka, Masataka and Dulikravich, George S},
  year={1998},
  publisher={Elsevier}
}

@article{hadamard1902problemes,
  title={Sur les probl{\`e}mes aux d{\'e}riv{\'e}es partielles et leur signification physique},
  author={Hadamard, Jacques},
  journal={Princeton University Bulletin},
  pages={49--52},
  year={1902},
  publisher={Princeton University}
}

@article{calvetti2018inverse,
  title={Inverse problems: From regularization to Bayesian inference},
  author={Calvetti, Daniela and Somersalo, Erkki},
  journal={Wiley Interdisciplinary Reviews: Computational Statistics},
  volume={10},
  number={3},
  pages={e1427},
  year={2018},
  publisher={Wiley Online Library}
}

@book{hofmann2013regularization,
  title={Regularization for applied inverse and ill-posed problems: a numerical approach},
  author={Hofmann, Bernd},
  volume={85},
  year={2013},
  publisher={Springer-Verlag}
}

@article{arridge2019solving,
  title={Solving inverse problems using data-driven models},
  author={Arridge, Simon and Maass, Peter and {\"O}ktem, Ozan and Sch{\"o}nlieb, Carola-Bibiane},
  journal={Acta Numerica},
  volume={28},
  pages={1--174},
  year={2019},
  publisher={Cambridge University Press}
}

@article{afkham2021learning,
  title={Learning regularization parameters of inverse problems via deep neural networks},
  author={Afkham, Babak and Chung, Julianne and Chung, Matthias},
  journal={Inverse Problems},
  volume={37},
  number={10},
  pages={105017},
  year={2021},
  publisher={IOP Publishing}
}

@article{genedy2023physics,
  title={Physics-informed neural networks for predicting liquid dairy manure temperature during storage},
  author={Genedy, Rana A and Chung, Matthias and Ogejo, Jactone A},
  journal={Neural Computing and Applications},
  volume={35},
  number={16},
  pages={12159--12174},
  year={2023},
  publisher={Springer}
}

@article{chung2024paired,
  title={Paired autoencoders for likelihood-free estimation in inverse problems},
  author={Chung, Matthias and Hart, Emma and Chung, Julianne and Peters, Bas and Haber, Eldad},
  journal={Machine Learning: Science and Technology},
  volume={5},
  number={4},
  pages={045055},
  year={2024},
  publisher={IOP Publishing}
}

@article{afkham2024uncertainty,
  title={Uncertainty quantification for goal-oriented inverse problems via variational encoder-decoder networks},
  author={Afkham, Babak and Chung, Julianne and Chung, Matthias},
  journal={Inverse Problems},
  volume={40},
  number={7},
  year={2024}
}

@article{pilozzi2018machine,
  title={Machine learning inverse problem for topological photonics},
  author={Pilozzi, Laura and Farrelly, Francis A and Marcucci, Giulia and Conti, Claudio},
  journal={Communications Physics},
  volume={1},
  number={1},
  pages={57},
  year={2018},
  publisher={Nature Publishing Group UK London}
}

@article{ongie2020deep,
  title={Deep learning techniques for inverse problems in imaging},
  author={Ongie, Gregory and Jalal, Ajil and Metzler, Christopher A and Baraniuk, Richard G and Dimakis, Alexandros G and Willett, Rebecca},
  journal={IEEE Journal on Selected Areas in Information Theory},
  volume={1},
  number={1},
  pages={39--56},
  year={2020},
  publisher={IEEE}
}

@article{kim2018geophysical,
  title={Geophysical inversion versus machine learning in inverse problems},
  author={Kim, Yuji and Nakata, Nori},
  journal={The Leading Edge},
  volume={37},
  number={12},
  pages={894--901},
  year={2018},
  publisher={Society of Exploration Geophysicists}
}

@article{riis2021computed,
  title={Computed tomography with view angle estimation using uncertainty quantification},
  author={Riis, Nicolai Andr{\'e} Brogaard and Dong, Yiqiu and Hansen, Per Christian},
  journal={Inverse Problems},
  volume={37},
  number={6},
  pages={065007},
  year={2021},
  publisher={IOP Publishing}
}

@article{hornik1989multilayer,
  title={Multilayer feedforward networks are universal approximators},
  author={Hornik, Kurt and Stinchcombe, Maxwell and White, Halbert},
  journal={Neural Networks},
  volume={2},
  number={5},
  pages={359--366},
  year={1989},
  publisher={Elsevier}
}

@article{cybenko1989approximation,
  title={Approximation by superpositions of a sigmoidal function},
  author={Cybenko, George},
  journal={Mathematics of Control, Signals, and Systems},
  volume={2},
  number={4},
  pages={303--314},
  year={1989},
  publisher={Springer}
}

@article{leshno1993multilayer,
  title={Multilayer feedforward networks with a nonpolynomial activation function can approximate any function},
  author={Leshno, Moshe and Lin, Vladimir Yaakov and Pinkus, Allan and Schocken, Shimon},
  journal={Neural Networks},
  volume={6},
  number={6},
  pages={861--867},
  year={1993},
  publisher={Elsevier}
}

@article{kiss20232detect,
    author = {Maximilian B. Kiss and Sophia B. Coban and K. Joost Batenburg and Tristan van Leeuwen, and Felix Lucka},
    title = {2DeteCT - A large 2D expandable, trainable, experimental Computed Tomography dataset for machine learning},
    journal = {Scientific Data},
    volume = {10},
    number = {576},
    year = {2023}
}

@article{wang2004image,
  author={Zhou Wang and Bovik, A.C. and Sheikh, H.R. and Simoncelli, E.P.},
  journal={IEEE Transactions on Image Processing}, 
  title={Image quality assessment: from error visibility to structural similarity}, 
  year={2004},
  volume={13},
  number={4},
  pages={600-612},
  doi={10.1109/TIP.2003.819861}}

\appendix

\end{document}